\documentclass[journal,compsoc]{IEEEtran}
\usepackage{amsfonts}
\usepackage{amssymb}
\newcommand*{\rom}[1]{%
  \textup{\uppercase\expandafter{\romannumeral#1}}%
}
\usepackage{dsfont}
\usepackage{array, makecell}
\usepackage{mathtools}
\usepackage{graphicx}
\usepackage{subfigure}
\usepackage{enumerate}
\usepackage{amsmath}
\usepackage{color}
\usepackage{amsthm}
\usepackage[bottom=1.18in,left=0.65in,right=0.70in,top=0.67in]{geometry}
\newtheorem{theorem}{Theorem}
\newtheorem{proposition}{Proposition}

\newtheorem{lemma}{Lemma}
\usepackage{algorithm}
\usepackage{algpseudocode}
\usepackage{stfloats}
\usepackage{multirow}
\theoremstyle{definition}
\usepackage{color}
\usepackage{url}
\allowdisplaybreaks[4]
\newcommand\reg{\textsc{Regret}}
\usepackage{amsthm}

\ifCLASSOPTIONcompsoc
  \usepackage[nocompress]{cite}
\else
  \usepackage{cite}
\fi

\begin{document}
%


\title{From Novice to Expert: Cost-Aware Bandits for Evolving Worker Performance in Crowdsensing}


\author{Yin Huang ~\IEEEmembership{Student Member,~IEEE}
           ~~Qingsong Liu~\IEEEmembership{Member,~IEEE}
       ~~Jie~Xu~\IEEEmembership{Senior Member,~IEEE}
\IEEEcompsocitemizethanks{\IEEEcompsocthanksitem Y. Huang and J. Xu are with the Department of Electrical and Computer Engineering, University of Florida.
Email: \{yin.huang, jie.xu\}@ufl.edu.

\IEEEcompsocthanksitem Q. Liu is with the Manning College of Information and Computer Sciences, University of Massachusetts Amherst. Email: qingsongliu@umass.edu.
}}

\IEEEtitleabstractindextext{%
\begin{abstract}

Mobile crowdsensing (MC) recruits mobile users to perform sensing tasks using their smartphones, enabling large-scale applications such as traffic monitoring and environmental sensing. A fundamental challenge is online worker recruitment under uncertainty, where the platform must learn workers' sensing performance while operating with a limited budget. Existing learning-based MC recruitment methods typically assume that each worker's sensing quality is stationary with a fixed mean over time. In practice, however, worker performance often improves with experience and eventually stabilizes, while the incurred sensing cost can be unknown in advance due to time-varying device and context states.
In this paper, we study a budget-constrained online recruitment problem in which the platform selects one worker in each round, observes the sensing quality and incurred cost, where the expected sensing quality of each worker increases with experience and eventually converges to a plateau, and repeats until the budget is exhausted. We formulate this problem as a structured bandit model where each worker's expected reward evolves according to an unknown increasing-then-converging function of its participation count, and each worker has an unknown expected cost. We develop a cost-aware online learning framework that jointly learns evolving reward trajectories and heterogeneous costs, detects performance saturation, and allocates the limited budget to maximize long-term sensing utility. We provide theoretical performance guarantees and validate the proposed approach through extensive experiments, demonstrating consistent improvements over baselines that ignore experience-driven dynamics or assume known costs.
\end{abstract}

}
\maketitle
\IEEEdisplaynontitleabstractindextext
\IEEEpeerreviewmaketitle
\IEEEraisesectionheading{\section{Introduction}}

\IEEEPARstart{M}obile crowdsensing (MC) has emerged as a powerful paradigm for collecting data from widespread locations via everyday mobile devices~\cite{ma2014opportunities}. By leveraging the sensors and connectivity of ubiquitous smartphones, MC enables large-scale sensing applications that would be infeasible for any single user. For instance, MC platforms have been used to monitor city traffic, measure urban noise and pollution levels, and map wireless network coverage. In a typical MC system, a central platform recruits a crowd of smartphone users (workers) to perform sensing tasks, aggregating their contributions to build rich spatiotemporal datasets. The success of such applications hinges on effective worker recruitment, selecting the right participants to maximize data quality under practical constraints like limited budgets and heterogeneous device capabilities.

Despite extensive research on MC recruitment and task allocation algorithms, most prior models assume that each worker’s sensing performance is fixed or stationary over time. In reality, however, participant performance can evolve as they gain experience. Just as crowdworkers on online platforms learn and improve with each task completed, MC participants may become more proficient at sensing tasks through repeated participation. Empirical evidence from crowdsourcing supports this learning effect: for example, the accuracy of Topcoder developers~\cite{wang2017recommending} was observed to increase significantly with the number of tasks completed before eventually stabilizing at a high level. We model this behavior by assuming that each worker’s expected reward follows an \emph{increasing-then-converging} trajectory, a structure that captures both early-stage learning and eventual performance saturation. This suggests a crucial gap in current crowdsensing frameworks: worker improvement dynamics are largely ignored. Failing to account for learning means existing approaches may undervalue novice workers who could become top-performers with more experience, or overcommit to individuals who have already reached their performance plateau.

While prior work has leveraged multi-armed bandit (MAB) frameworks for modeling worker selection in MC systems \cite{chen2022learning,su2025crowdsensing,zhou2024unknown}, these formulations typically assume that each worker’s quality is stationary and represented by a fixed, unknown reward parameter. {In this work, we consider a mobile crowdsensing system with a small, fixed pool of participants in which each task is assigned to one worker, a setting motivated by short-term campaigns built from location-specific microtasks, such as street-issue verification (e.g., potholes or broken street lights), in-store retail audits (e.g., price checks or on-shelf availability checks), and localized environmental measurements (e.g., noise or air-quality readings), where the platform naturally operates over a bounded set of currently available participants and each assignment corresponds to a single report, visit, or measurement \cite{li2020consensus,carrera2013people,maisonneuve2009noisetube,tong2016online,chen2022learning}.} In each round, the platform assigns the sensing task to a worker, observes the obtained sensing quality and the incurred cost, and then updates its recruitment policy for subsequent rounds; the process terminates when the cumulative cost exhausts the total budget. Under this sequential decision process, the platform must balance exploration and exploitation since both workers’ sensing quality and recruitment cost are initially unknown. The budget constraint further amplifies the cost of exploration because every trial consumes limited resources. Moreover, workers’ sensing quality is non-stationary and personalized: a worker may improve with experience during early participation and then gradually stabilize, while different workers can exhibit different learning speeds and plateau levels. Finally, we also consider unknown costs in a realistic manner. In practice, a worker’s effective sensing cost depends on its instantaneous device and context state, such as battery level, network conditions, and surrounding environment; thus, even if the platform adopts a pre-specified cost function, the worker-specific cost parameter is not known a priori and is estimated by the worker in each round based on these factors, making the platform learn the expected cost online.

In this paper, we formulate the online worker selection problem as a structured bandit setting in which each worker’s expected sensing quality evolves according to an unknown, increasing-then-converging function, and expected cost incurs a fixed but unknown cost. The goal is to allocate a finite task budget across workers to maximize the cumulative qualities obtained. This setting requires not only resolving uncertainty and learning heterogeneous costs, but also modeling each worker’s full learning trajectory, including their initial skill level, rate of improvement, and convergence point. To address this, we develop a cost-sensitive extension of the \emph{Time-Increasing Upper Confidence Bound} (TI-UCB) algorithm, called CATI-UCB, which is designed to operate under structured reward dynamics and budget constraints. CATI-UCB combines three core components: it uses a reward-cost ratio to guide exploration and exploitation; it fits an online linear model to estimate each worker’s early learning behavior; and it employs change-point detection to identify when learning saturates. These components work together to adaptively prioritize workers based on their long-term efficiency rather than short-term reward. We prove that CATI-UCB achieves sublinear regret relative to the optimal policy, and we empirically demonstrate that it significantly outperforms existing baselines that ignore temporal reward structure or assume stationarity.

In summary, our main contributions are as follows:
\begin{itemize}
    \item We formulate a new online worker selection problem in crowdsourcing where each worker’s expected reward follows an unknown, increasing-then-converging function, and expected cost incurs a fixed but unknown cost. This setting captures realistic worker learning dynamics and introduces new algorithmic challenges beyond standard stationary bandit models.
    \item We propose CATI-UCB, a structured bandit algorithm that jointly addresses uncertainty, cost-awareness, and non-stationary reward dynamics. The algorithm estimates each worker’s learning curve via online linear regression, detects performance saturation through change-point detection, and selects workers based on an upper confidence bound of the reward-to-cost ratio.
    \item We prove that CATI-UCB achieves logarithmic regret compared with the optimal policy under non-stationary rewards.
    \item We evaluate CATI-UCB through extensive synthetic experiments that model realistic worker learning patterns. The results show that CATI-UCB consistently outperforms baseline methods.
\end{itemize}

\section{Related work}
{\textbf{Mobile Crowdsensing and worker recruitment:} Mobile crowdsensing (MC) studies how to recruit mobile users to perform sensing tasks under uncertainty, and a large body of work has formulated worker recruitment as an online learning problem using multi-armed bandits or related frameworks to handle unknown worker quality, limited budget, and incentive constraints \cite{gao2019unknown,gao2021budgeted,zhao2020differentially,su2025crowdsensing,zhou2024unknown,tang2023btv,ouyang2025mwrs,song2025unknown}. 
Recent studies have further broadened the scope of MC recruitment, for example by addressing insufficient participation through social-network-assisted recruitment~\cite{wang2020socialrecruiter} or by studying dynamic online dispatch under time-varying resource-quality tradeoffs~\cite{chen2022learning}. 
This line of work typically studies sequential server-assigned recruitment under newly observed feedback, which is also the setting considered in this paper.
However, most existing learning-based MC formulations either assume stationary worker quality or focus on other uncertainties such as truth discovery, trust/reputation, requester-side uncertainty, or incentive design \cite{gao2019unknown,gao2021budgeted,tang2023btv,ouyang2025mwrs,song2025unknown}. 
In contrast, our work focuses on a different source of non-stationarity: a worker's expected sensing reward improves with repeated participation and eventually saturates. Motivated by empirical evidence on skill improvement \cite{wang2017recommending}, we model each worker's expected reward as an \emph{increasing-then-converging} function of participation count, while simultaneously learning worker costs online under a finite budget.}

\textbf{Non-stationary Bandits:}
The increasing-then-converging reward trend studied in our work is closely related to the non-stationary bandit literature, which addresses changing reward distributions over time~\cite{huang2023adversarial,huang2025learning}. Existing approaches typically assume either piecewise-stationary~\cite{garivier2011upper, cao2019nearly, besbes2014stochastic} or smoothly-varying~\cite{besbes2014stochastic, russac2019weighted} rewards, and adapt through sliding windows, change-point detection, or discounting techniques. More recently, rested bandits~\cite{tekin2012online,seznec2020single, levine2017rotting} have been proposed, where rewards depend on how often an arm is pulled, capturing trends like skill acquisition. Some works model monotonic or increasing reward patterns~\cite{heidari2016tight, metelli2022stochastic}, but most of these methods assume simplified or deterministic trends and do not explicitly model the increasing-then-converging pattern observed in skill-based tasks. The recent work~\cite{xia2024llm} considers an increasing-then-converging structure similar to ours, but focuses on model selection and does not address budget constraints. In contrast, our work integrates the increasing-then-converging structure with cost-sensitive online learning under a finite budget, enabling more efficient resource allocation.

\textbf{Bandits with Knapsack:}
``Bandits with knapsack'' (BwK) extends the multi-armed bandit problem to settings with limited resources, aiming to maximize total rewards under a budget constraint~\cite{badanidiyuru2012learning,badanidiyuru2018bandits,huang2024adversarial}. Applications include dynamic pricing~\cite{babaioff2015dynamic}, procurement~\cite{singla2013truthful}, and pay-per-click ad allocation~\cite{combes2015bandits}. BwK research can be broadly categorized into stochastic and adversarial settings. In the stochastic case, where each arm follows a fixed but unknown distribution~\cite{badanidiyuru2012learning,agrawal2014bandits,agrawal2016efficient,li2021symmetry}, optimal regret bounds have been achieved by successive elimination~\cite{badanidiyuru2012learning}, UcbBwK~\cite{agrawal2014bandits}, and primal-dual algorithms~\cite{li2021symmetry}. In the adversarial case, rewards can be manipulated by an adversary~\cite{immorlica2022adversarial,rangi2018unifying}. Recent work also explores non-stationary BwK~\cite{liu2022non}. However, these methods generally do not account for the increasing-then-converging patterns observed in skill-based tasks, leading to suboptimal performance. In contrast, our CATI-UCB explicitly models this trend and adaptively detects convergence points under budget constraint, outperforming prior approaches both theoretically and empirically in online worker recruitment.

\section{System Model}\label{sec:problem_inc}
{We consider a mobile crowdsensing (MC) setting in which a requester launches a short-term sensing campaign with a fixed budget, and a platform is responsible for coordinating task assignments among a fixed pool of $K$ committed and currently available mobile workers, as illustrated in Fig.~\ref{fig:main_procedure}. The campaign proceeds in rounds. In each round, the platform selects one worker to execute one sensing job (e.g., capturing traffic photos at an intersection, measuring ambient noise, or collecting WiFi fingerprints), based on the feedback accumulated from previous assignments. After the selected worker completes the sensing job and uploads the sensed data, the platform evaluates the outcome and observes two signals: a reward that reflects the worker's current sensing performance (e.g., data validity, clarity, timeliness, or compliance with task requirements), and the incurred cost. In our primary interpretation, this cost represents the realized execution delay (or task completion time), normalized to $[0,1]$; more generally, it can also absorb other execution-related burdens such as network usage through a scalar cost definition. The platform then updates its internal selection strategy and proceeds to the next round. This interactive process continues until the total budget is exhausted. Under this interpretation, the budget can be viewed as a deadline-like cumulative delay budget, and the platform's objective is to allocate the limited budget across workers over time so as to maximize the cumulative sensing utility before this budget is consumed. We do not explicitly model worker refusals, arrivals, or departures in this paper; instead, we focus on a fixed worker pool so as to isolate the online learning problem induced by evolving worker performance and unknown costs. This setup captures practical scenarios such as sustained urban monitoring, iterative participant onboarding and calibration, and continuous improvement of sensing quality on real-world MC platforms.}

\begin{figure}[ht]
\centerline{\includegraphics[width=75mm]{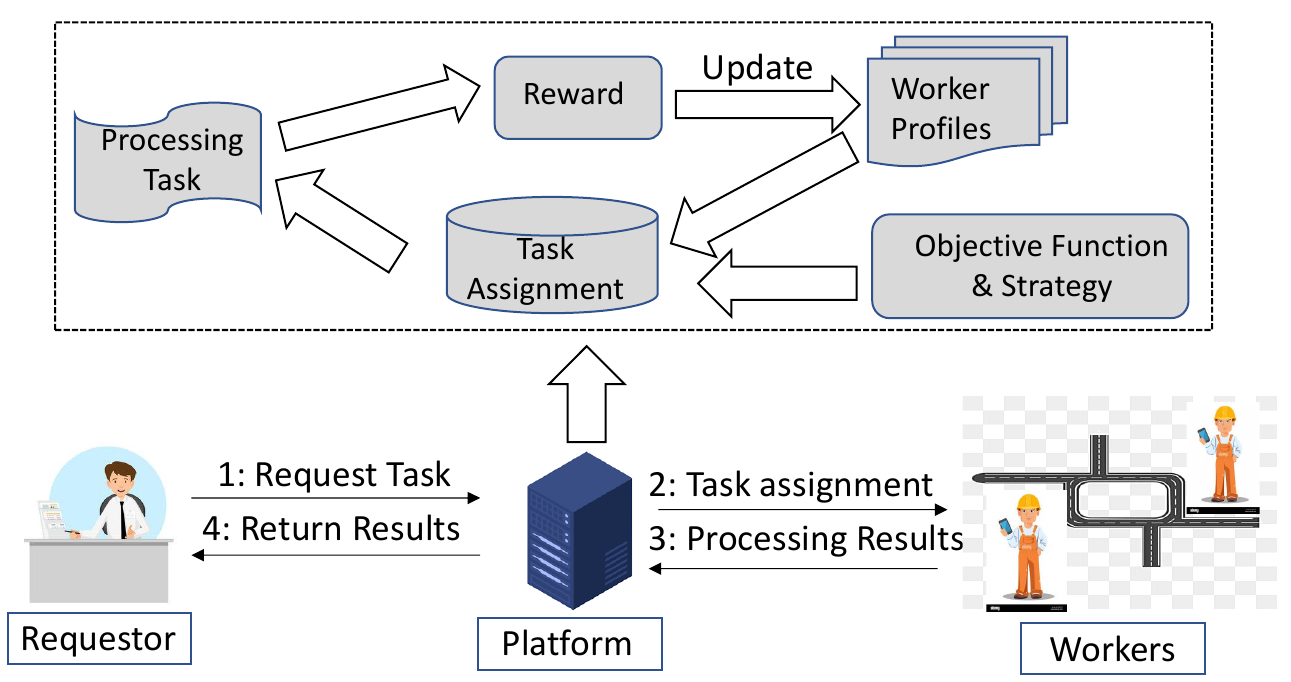}}
\caption{{\it Illustration of the main procedures in the mobile crowdsensing platform}}
\label{fig:main_procedure}
\end{figure}

\textbf{Online Worker Recruitment:}
The online worker recruitment involves a finite set of candidate workers denoted by $\{1,2,\dots,K\}$, where worker $i$ denotes the $i$-th participant among $K$ workers. Time is slotted into rounds indexed by $t=1,2,\dots$. In each round $t$, the platform assigns the sensing job to exactly one worker $A_t \in [K]$ based on the feedback accumulated from previous assignments (i.e., observed sensing quality and incurred cost), and then updates its selection strategy for subsequent rounds.

\textbf{Reward \& Cost Feedback:}
After worker $A_t$ completes the sensing job and uploads the sensed data, the platform evaluates the returned data using a quality assessment pipeline (e.g., completeness, timeliness, and cross-validation with redundant measurements when available), and obtains a normalized quality score $r_{A_t,t}\in[0,1]$ as the reward. Meanwhile, the platform observes the realized cost $Y_{A_t,t}\in[0,1]$ associated with recruiting worker $A_t$, {which primarily represents realized execution delay, and more generally may also absorb other execution-related burdens such as bandwidth consumption through a scalar cost definition.}

We model $(r_{i,t},Y_{i,t})$ as stochastic feedback. Specifically, conditional on worker $i$'s experience level, the reward observations are independent draws from an unknown distribution with a mean that evolves with experience. Let $N_{i,t}=\sum_{s=1}^{t}\mathbf{1}\{A_s=i\}$ denote the number of times worker $i$ has been selected up to the end of round $t$. We write the expected reward as
$\mu_i(t)=\tilde{\mu}_i(N_{i,t})$,
where the unknown function $\tilde{\mu}_i(n)$ is nondecreasing in $n$ and converges to a plateau, capturing the learning-by-doing effect commonly observed in crowd work settings~\cite{wang2017recommending}.

For the cost, we assume that worker $i$ has an unknown, worker-specific expected cost $c_i$, and the observed costs $\{Y_{i,t}\}$ are independent draws from an unknown distribution with mean $c_i$ and support in $[0,1]$. This captures practical MC deployments where the realized sensing cost depends on the worker's instantaneous device/context state (e.g., battery level, network condition, and environment factors). As a result, even if the platform specifies a cost model, the effective cost parameter is not known a priori and must be learned online from observations.

We denote by $\nu_i$ the \emph{convergence point} of worker $i$, representing the number of tasks needed for the worker’s performance to stabilize, and by $\mu_i^*$ the corresponding \emph{expected converged reward}, equal to $\tilde{\mu}_i(\nu_i)$. The value of $\nu_i$ varies across workers, reflecting different learning rates. The shape of the reward trajectory $\tilde{\mu}_i(n)$ may vary in form and can be approximated using different function classes. In our work, we adopt a \emph{piece-wise linear} approximation of $\tilde{\mu}_i(n)$ to preserve the structure of increasing-then-converging reward while enabling efficient analysis and algorithmic design, which is shown as Fig 2. Specifically, we assume that each worker’s reward initially increases linearly with the number of tasks, then stabilizes at $\mu_i^*$ once learning saturates within the budget. This approximation is characterized by three parameters:(1) the initial reward level (reflecting prior ability), $\tilde{\mu}_i(0)$, (2) the convergence point $\nu_i$ (when learning plateaus), and (3) the learning rate (how quickly performance improves) $\frac{\mu_i^* - \tilde{\mu}_i(0)}{\nu_i}$. This model captures key trends in worker learning behavior. Furthermore, our experiments show that the algorithm performs well even when the true reward follows other forms (e.g., \textbf{negative exponential}), outperforming baselines across settings.

\begin{figure}
    \centering
    \includegraphics[width=0.25\textwidth]{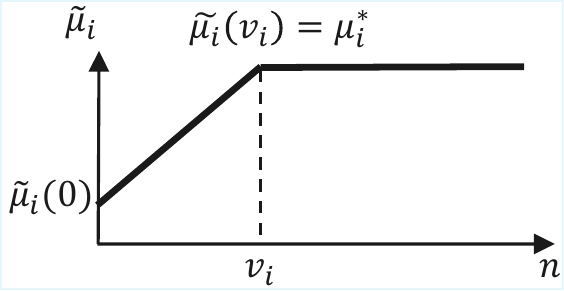}
    \caption{{\it Piece-wise Linear Function}}
    \label{fig:piese-wise}
\end{figure}
\textbf{Budget Constraint:}
After observing the reward cost pairs ${(r_{A_t,t},Y_{A_t,t})}$, the platform then correspondingly updates its selection policy for the next round of worker selection and the total budget is decreased by the cost produced by the selected worker. Denote $B$ as the total budget, then the stopping time ${T_{\mathcal{A},B}}$ for the selecting algorithm $\mathcal{A}$ can be characterized by:  $\sum_{t=1}^{T_{\mathcal{A},B}-1}Y_{A_t,t} \le B < \sum_{t=1}^{T_{\mathcal{A},B}}Y_{A_t,t}.$

\textbf{Objective:}
The goal of the platform is to efficiently select the most suitable worker being selected, i.e., yielding the highest cumulative reward until the budget is exhausted. Specifically, let the total reward accumulated by the selection algorithm $\mathcal{A}$ up to time $T_{\mathcal{A},B}$ be denoted as $R_{\mathcal{A},B} = \sum_{t=1}^{T_{\mathcal{A},B}}r_{A_t,t}$, the platform's goal can be
formulated as follows:
\begin{align}
    &\max R_{\mathcal{A},B} = \sum_{t=1}^{T_{\mathcal{A},B}}r_{A_t,t} \quad s.t. \sum_{t=1}^{T_{\mathcal{A},B}-1}Y_{A_t,t} \le B < \sum_{t=1}^{T_{\mathcal{A},B}}Y_{A_t,t}. \nonumber
\end{align}
Let $\{q_{i,B}\}_{i=1}^K$ be a valid worker-selection allocation under budget $B$ such that $\sum_{i=1}^K c_i q_{i,B} \le B$ where $q_{i,B}$ is the number of times worker $i$ is selected. Then the total reward for this allocation is $R = \sum_{i=1}^K \sum_{t=1}^{q_{i,B}} \mu_i(t)$. Therefore, the optimal expected reward $R^*(B)$ under known reward and cost functions is the solution to the following combinatorial optimization problem:
\begin{equation}
    R^*(B) := \max_{\{q_{i,B} \ge 0\}_{i=1}^K} \left\{ \sum_{i=1}^K \sum_{n=1}^{q_{i,B}} \tilde{\mu}_i(n) \;\middle|\; \sum_{i=1}^K c_i q_{i,B} \le B \right\}.
\end{equation}
This formulation characterizes the best achievable cumulative reward when full knowledge of reward trajectories and costs is available, and switching between workers is allowed. We use expected regret to evaluate the performance of the algorithm $\mathcal{A}$, which is defined below.
\begin{align}
        \reg(\mathcal{A},B) &= R^*(B) -  \mathbb{E}[R_{\mathcal{A},B}] \nonumber
        \\&= R^* - \sum_{i=1}^{K}\mathbb{E}[\sum_{s=1}^{N_{i,T_{\mathcal{A},B}}}\mu_{i,s}],   \label{eq:cum_regret} 
\end{align}
where $N_{i,T_{\mathcal{A},B}}$ is the actual number of times that worker $i$ has been selected by the platform till ending time step $T_{\mathcal{A},B}$, {$\mu_{i,s}$ denotes the expected reward of worker $i$ at its $s$-th selection/pull,} and the $\mathbb{E}$ is taken over the randomness of rewards, costs, and the selecting algorithm.

In our setting, the values of~$\tilde{\mu}_i(n)$, $\nu_i$, $\mu_i^*$, and the cost~$c_i$ of selecting worker~$i$ are unknown and must be learned online. This departs from classical bandit models in crowdsourcing, which assume stationary rewards and represent each worker with a single scalar. Such simplifications overlook the non-stationary, structured reward evolution of real workers, whose performance improves then stabilizes. Capturing this increasing-then-converging behavior poses new challenges: the platform must identify high-performing workers, anticipate their learning, detect convergence, and account for heterogeneous costs under a finite budget. Existing algorithms that assume stationarity or ignore this non-stationary trend are thus inadequate, motivating a structured, cost-sensitive, non-stationary framework.

\textbf{Remark:} We focus on tasks from the same category/type in this work, but our framework can be extended to heterogeneous tasks by incorporating task types/features (e.g., letting reward/cost depend on the current task context). For a finite number of task types, one can maintain separate learning-curve/cost estimates per (worker, type) and apply the same cost-aware UCB selection rule conditioned on the task type. 

Before we dive into the learning algorithm development, let us have a better understanding of the optimal offline policy. 
\begin{lemma}[Asymptotic structure of the offline optimum]
Let
\[
\rho^\star := \max_{i \in [K]} \frac{\mu_i^\star}{c_i},
\qquad
i^\dagger \in \arg\max_{i \in [K]} \frac{\mu_i^\star}{c_i}.
\]
Then there exists a constant $C_0 > 0$, independent of the budget $B$, such that
\[
\rho^\star B - C_0 \le R^\star(B) \le \rho^\star B, \qquad \forall B \ge 0.
\]
In particular, $R^\star(B) = \rho^\star B + O(1).$
Moreover, the single-worker policy that allocates all available budget to worker $i^\dagger$ is optimal up to an additive constant independent of $B$.
\end{lemma}
\begin{proof}
Recall that
\[
R^\star(B)
=
\max_{\{q_i \in \mathbb{Z}_{\ge 0}\}_{i=1}^K}
\left\{
\sum_{i=1}^K \sum_{n=1}^{q_i} \tilde{\mu}_i(n)
\,\middle|\,
\sum_{i=1}^K c_i q_i \le B
\right\}.
\]
We first prove the upper bound. For any feasible allocation $\{q_i\}_{i=1}^K$, since
$\tilde{\mu}_i(n) \le \mu_i^\star$ for all $i$ and $n$, we have
\[
\sum_{i=1}^K \sum_{n=1}^{q_i} \tilde{\mu}_i(n)
\le
\sum_{i=1}^K \mu_i^\star q_i
=
\sum_{i=1}^K \frac{\mu_i^\star}{c_i} c_i q_i
\le
\rho^\star \sum_{i=1}^K c_i q_i
\le
\rho^\star B.
\]
Taking the maximum over all feasible allocations yields $R^\star(B) \le \rho^\star B.$
Next, we prove the lower bound by considering the policy that allocates the budget only to worker $i^\dagger$. Let $q^\dagger := \left\lfloor \frac{B}{c_{i^\dagger}} \right\rfloor.$
This is feasible since $c_{i^\dagger} q^\dagger \le B$. Under the piecewise linear increasing-then-converging model, worker $i^\dagger$ reaches the plateau $\mu_{i^\dagger}^\star$ after $\nu_{i^\dagger}$ selections, and hence
\[
\tilde{\mu}_{i^\dagger}(n) = \mu_{i^\dagger}^\star,
\qquad \forall n \ge \nu_{i^\dagger}.
\]
Therefore, we have $\sum_{n=1}^{q^\dagger} \tilde{\mu}_{i^\dagger}(n)
\ge
\mu_{i^\dagger}^\star (q^\dagger - \nu_{i^\dagger})_+.$
Using $q^\dagger \ge \frac{B}{c_{i^\dagger}} - 1$, we obtain
\[
\sum_{n=1}^{q^\dagger} \tilde{\mu}_{i^\dagger}(n)
\ge
\mu_{i^\dagger}^\star
\left(
\frac{B}{c_{i^\dagger}} - 1 - \nu_{i^\dagger}
\right)
=
\rho^\star B - \mu_{i^\dagger}^\star (1 + \nu_{i^\dagger}).
\]
Since this reward is achievable by a feasible policy, it follows that $R^\star(B)
\ge
\rho^\star B - \mu_{i^\dagger}^\star (1 + \nu_{i^\dagger}).$

Thus, by setting $C_0 := \mu_{i^\dagger}^\star (1 + \nu_{i^\dagger}),$
we conclude that
\[
\rho^\star B - C_0 \le R^\star(B) \le \rho^\star B,
\qquad \forall B \ge 0.
\]
Hence,
\[
R^\star(B) = \rho^\star B + O(1).
\]
The last statement follows immediately from the lower-bound construction above.
\end{proof}

\textbf{Remark:}
Lemma~1 provides a uniform finite-budget characterization of the
offline optimum and does not require all workers to reach their
plateaus before the budget is exhausted. The upper
bound only relies on
$\tilde{\mu}_i(n)\le \mu_i^\star$ for all $i$ and $n$.
For the lower-bound construction, let $q^\dagger
=
\left\lfloor \frac{B}{c_{i^\dagger}} \right\rfloor.$
If $q^\dagger < \nu_{i^\dagger}$, then
$(q^\dagger-\nu_{i^\dagger})_+=0$, and hence the lower bound
remains valid, although it may be loose. Thus, no assumption of the
form $\nu_i \ll T_{\mathcal{A},B}$ is required.

The additive constant $C_0$ accounts for the finite reward deficit
incurred during the pre-convergence phase. Consequently, $R^\star(B)=\rho^\star B+O(1),
\qquad
\rho^\star
=
\max_{i\in[K]}\frac{\mu_i^\star}{c_i}.$
Therefore, allocating the budget to any worker $i^\dagger
\in
\arg\max_{i\in[K]}\frac{\mu_i^\star}{c_i}$
is asymptotically optimal up to an additive constant. This
large-budget interpretation is a consequence of Lemma~1 rather than
an additional assumption required by the lemma.

\section{Algorithm Design}\label{sec:algorithm}
In Section \ref{sec:tiucb_alg}, we propose the Cost-aware Time-increasing UCB (CATI-UCB) algorithm for the online worker recruitment problem formulated in Section \ref{sec:problem_inc}. 

\subsection{Overview of the CATI-UCB Algorithm}
To address the challenges of unknown, non-stationary rewards and heterogeneous costs, we propose the CATI-UCB (Algorithm~\ref{alg:TI-UCB}). Unlike standard bandit approaches that model each worker with a fixed reward parameter, our setting explicitly accounts for each worker’s initial skill, learning rate, and convergence point~$\nu_i$, defining the increasing-then-converging reward trajectory~$\tilde{\mu}_i(n)$. Combined with the unknown cost~$c_i$, these attributes determine each worker’s long-term efficiency under budget constraints. CATI-UCB leverages this structure through three components: (i) modeling the increasing reward trend with online piecewise linear regression to predict short-term improvement; (ii) detecting when learning stabilizes via change-point detection; and (iii) selecting workers based on cost-aware upper-confidence estimates of their reward-to-cost ratio. These components interact in a loop to dynamically balance exploration and exploitation until the budget is exhausted, prioritizing long-term over short-term gains.

\subsection{CATI-UCB Algorithm}\label{sec:tiucb_alg}
At the initial phase of the CATI-UCB, each worker is selected once to get the initial estimation value of the reward and cost.

\subsubsection{Increasing Reward Prediction}
To estimate a worker's initial skill and learning rate, CATI-UCB models the early-stage reward trajectory using a simple linear approximation:
\begin{equation}
    \bar{\mu}_{i,N_{i,t}} = \hat{a}_{i,N_{i,t}} \cdot N_{i,t}+\hat{b}_{i,N_{i,t}}\label{eq:least_square}
\end{equation}
where $N_{i,t}$ is the number of times worker $i$ has been selected up to time $t$, and the parameters $\hat{a}_{i,N_{i,t}}$ and $\hat{b}_{i,N_{i,t}}$ are updated via least squares regression over historical reward observations. Here, $\hat{b}_{i,N_{i,t}}$ captures the estimated initial reward (reflecting the worker's prior ability), while $\hat{a}_{i,N_{i,t}}$ estimates the rate of improvement over time (reflecting the worker's learning rate).



This predictive model enables the algorithm to anticipate future rewards based on early performance, allowing more informed exploration decisions. When a worker is detected to have converged (as indicated in line 11 of Algorithm~1), the model fit is reset by setting $N_{i,t} = 1$ and restarting the collection of data points used to estimate $\hat{a}_{i,N_{i,t}}$ and $\hat{b}_{i,N_{i,t}}$. This ensures that the linear model continues to reflect only the pre-convergence trend, maintaining an accurate and adaptive estimate of each worker’s learning dynamics.

\subsubsection{Worker Selection Design}
We rank worker~$i$ by the ratio of the upper confidence bound (UCB) of its estimated reward to the lower confidence bound (LCB) of its estimated cost:~$\frac{\hat{\mu}_{i,t}}{\hat{c}_{i,t}}$, following the Optimism in the Face of Uncertainty (OFU) principle. Although UCB is a biased estimator, it converges asymptotically. At step $t$, the confidence padding term is~$\epsilon_{i,t} = 16\sqrt{\frac{2\ln(1/\delta)}{N_{i,t}}}$, yielding~$\hat{\mu}_{i,t} = \bar{\mu}_{i,t} + \epsilon_{i,t}$ and~$\hat{c}_{i,t} = \bar{c}_{i,t} - \epsilon_{i,t}$. Here, $\bar{\mu}_{i,t}$ is computed via Eq.\eqref{eq:least_square}, and~$\bar{c}_{i,t}$ is the average cost. 

With the prediction of the increasing reward, CATI-UCB seeks to balance exploration and exploitation by adding an uncertainty term to the predicted reward of each worker as Line 5 of Algorithm \ref{alg:TI-UCB}. The algorithm then selects the worker with the maximum value of $\frac{\hat{\mu}_{i,t}}{\hat{c}_{i,t}}$ to play, receives a reward $r_{A_{t},t}$ and cost $Y_{A_t,t}$ and update the observation records of worker $i$ as described in Line 6-8 of Algorithm \ref{alg:TI-UCB}. 

The concentration level of the reward’s uncertainty term follows Eq.~\eqref{UCB bound}. To simplify the notation, we use $\hat{\mu}_{i,N_{i,t}}$ as $\hat{\mu}_i(t)$, and $\bar{\mu}_{i,N_{i,t}}$ as $\bar{\mu}_i(t)$ in the rest of the paper. Then we adopt the following confidence bound:
\begin{equation}\label{UCB bound}
    \hat{\mu}_i(t-1) =
    \left\{
    \begin{aligned}
        &\infty,  &&{ \text{if }\, N_{i,t-1} = 0}\,  \\
        &\bar{\mu}_i(t-1) + 16\sqrt{\frac{2\ln(1/\delta)}{N_{i,t-1}}}, && {\text{otherwise,}}
    \end{aligned}
    \right.
\end{equation}

According to Proposition 1 of ~\cite{xia2024llm}, for any $\delta\in(0, 1)$, $\mu_i(t) \leq \bar{\mu}_i(t) + 16\sqrt{\frac{2\ln(1/\delta)}{N_{i,t}}}$ holds with probability at least $1-\delta$. 

\subsubsection{Change Detection}
To identify when a worker's learning has saturated, CATI-UCB incorporates a change detection module to estimate the convergence point $\nu_i$ for each worker $i$. After a certain number of selections, different workers stabilize at distinct time steps, each reaching a steady reward level~$\mu_i^*$ for~$i \in [K]$. To estimate the convergence point~$\nu_i$ from the observed rewards of each worker~$i$, we track their temporal evolution using two overlapping windows of length~$\omega$. The windows slide as more rewards are collected for worker~$i$.

At time~$t+1$, we compare the average rewards in the two windows: the previous window~$w_1 = [N_{i,t}-2\omega+1, N_{i,t}-\omega]$ and the current window~$w_2 = [N_{i,t}-\omega+1, N_{i,t}]$, denoted by~$\bar{\mu}_{w_1, i}(t+1)$ and~$\bar{\mu}_{w_2, i}(t+1)$. If their difference exceeds~$\gamma/2$ (Line 10, Algorithm~\ref{alg:TI-UCB}), a change point~$\tau_i$ is detected, the observations for worker~$i$ are reset, and the current time step~$\tau_i$ is recorded. Otherwise, the algorithm proceeds to select workers and monitor new rewards. This mechanism adaptively distinguishes learning from stabilization, ensuring that the linear model in the increasing phase remains accurate and up-to-date. The statistical guarantee is given in Proposition~\ref{detect standard}.

\begin{proposition}(Proposition 2 of~\cite{xia2024llm})
\label{detect standard}
The reward change point of worker $i$ is considered to be reached, if $$|\bar{\mu}_{w_1, i}(t+1) - \bar{\mu}_{w_2, i}(t+1)|>{\gamma}/{2}\;,$$
where $\bar{\mu}_{w_1, i}(t+1)$ and $\bar{\mu}_{w_2, i}(t+1)$ are the predicted rewards for worker $i$ at time $t+1$ calculated by observations in the window $w_1 = \left[N_{i,t}-2\omega + 1,N_{i,t}-\omega\right]$, and by observations in the window $w_2 = \left[N_{i,t}-\omega + 1,N_{i,t}\right]$. With $\gamma \leq \sqrt{\frac{2}{\omega} (14+\frac{12}{|\omega -1|})^2 \ln(\frac{2}{\delta})}$, the above change detection inequality is valid with probability $1-\delta$.
\end{proposition}

\begin{algorithm}[h]
   \caption{CATI-UCB}
   \label{alg:TI-UCB}
   \begin{algorithmic}[1]

   \State \textbf{Input}: $K,\delta, \text{window size}$ $\omega, t=0$,.
   
   \State \textbf{Initialization}: Select each worker $i$ once, observe the $X_{i}$ and $Y_{i}$ to predict the  $\hat{\mu}_{i,t}$ and $\hat{c}_{i,t}$.

   \While{$\sum_{\tau = 1}^t c_{A_{\tau},\tau} \le B$}

   \State $t = t + 1, \epsilon_{i,t} = 16\sqrt{\frac{2\log 1/\delta}{N_{i,t}}}$

   \State \begin{align*}
       &A_t =\arg\max_{i}\{\frac{\hat{\mu}_{i,t} + \epsilon_{i,t}}{\max\{\hat{c}_{i,t}- \epsilon_{i,t},0.001\}}\},
   \end{align*}

   \State \text{Choose worker} $A_t$, and Observe reward $X_{A_t,t}, Y_{A_t,t}$.

   \State Update reward estimation $\hat{\mu}_{A_t,N_{A_t,t}}$, and cost estimation $\hat{c}_{A_t,N_{A_t,t}}$.

   \State $N_{A_t,t} = N_{A_t,t} + 1$.

   \If{$N_{A_t,t} \ge 2w$}

   \If{$|\hat{\mu}_{A_t,w_1}-\hat{\mu}_{A_t,w_2}| \ge \frac{1}{2}\sqrt{\frac{2}{w}{(14 + \frac{12}{|w-1|})}^2\log(2/\delta)}$}

   \State $\tau_{A_t} = t, N_{A_t,t} = 1$

   \EndIf
   \EndIf

   \EndWhile

\end{algorithmic}   
\end{algorithm}

\section{Regret Analysis}
In this section, we provide the regret upper bound of CATI-UCB in a typical increasing-then-converging reward within a given constrained budget setting.

Define the optimal worker as~$\arg\max_{i \in [K]} \frac{\mu_i^*}{c_i}$, assuming all workers have true costs above~$c_{\min}$. The suboptimal gap for worker~$i$ is~$\Delta_i = \frac{\mu_{i^\dagger}^*}{c_{i^\dagger}} - \frac{\mu_{i}^*}{c_{i}}$. Then, by Eq.\ref{eq:cum_regret}, the regret of the increasing-then-converging bandit under budget constraint using algorithm $\mathcal{A}$ can be rewritten as in Lemma~\eqref{lemma 2}.
\begin{lemma}\label{lemma 2}
	The expected regret can be characterized by
	\begin{align*}
		\reg(\mathcal{A},B) \le \sum_{i \neq i^\dagger} c_i \Delta_i \mathbb{E}[N_{i,T_{\mathcal{A},B}}]+\mu^*_{i^\dagger}/c_{i^\dagger} + \sum_{i\in [N]}\nu_i \mu^*_i.
		\end{align*}
\end{lemma}
\begin{proof}
\renewcommand{\qedsymbol}{}
Note that the stopping time $T_{\mathcal{A},B}$, defined by $\sum_{t=1}^{T_{\mathcal{A},B}-1} Y_{A_t,t} \leq B < \sum_{t=1}^{T_{\mathcal{A},B}} Y_{A_t,t}$, is a random variable depending on $B$. Then the proof of Lemma 2 consists of two steps:

\textbf{Step 1.}
Let $B_t$ denote the remaining budget at round $t$, and $\mathbf{I}(B_t \geq 0)$ the indicator of a valid round. Let $\mathcal{M}$ be the optimal policy that achieves $R^*$, the optimal expected total reward with known reward and cost distributions, and let $A_t^*$ denote its selected action. We first show that $R^* \leq (B+1)\frac{\mu^*_{i^\dagger}}{c_{i^\dagger}}$.
\begin{align*}
R^* &= \mathbb{E}\left[ \sum_{t=1}^{\infty} r_{A^*_t, t} \mathbf{I}(B_t \geq 0) \right]
\le \mathbb{E}\left[\sum_{t=1}^{\infty}\mu_{A_t^*}^*\mathbf{I}(B_t \geq 0)\right]
\\&\le \mathbb{E}\left[\sum_{t=1}^{\infty}c_{A_t^*}\frac{\mu_{i^\dagger}^*}{c_{i^\dagger}}\mathbf{I}(B_t \geq 0)\right]
=\frac{\mu_{i^\dagger}^*}{c_{i^\dagger}}\mathbb{E}\left[\sum_{t=1}^{\infty}c_{A_t^*}\mathbf{I}(B_t \geq 0)\right]
\\&=\frac{\mu_{i^\dagger}^*}{c_{i^\dagger}}\mathbb{E}\left[\sum_{t=1}^{\infty}c_{A_t^*,t}\mathbf{I}(B_t \geq 0)\right]
\leq (B + 1) \frac{\mu_{i^\dagger}^*}{c_{i^\dagger}},
\end{align*}
where the first inequality holds since $\mu_{i,t} \le \mu_{i}^*$ and the second follows from the suboptimal worker definition: $\frac{\mu_{A^*_t}^*}{c_{A^*_t}} \le \frac{\mu_{i^\dagger}^*}{c_{i^\dagger}} $.

\textbf{Step 2.}
According to step 1, the optimal reward can be upper bounded as $R^* \leq (B + 1) \frac{\mu^*_{i^\dagger}}{c_{i^\dagger}} < \mathbb{E} \left[ \sum_{t=1}^{T_{\mathcal{A},B}} c_{A_t, t} + 1 \right] \frac{\mu^*_{i^\dagger}}{c_{i^\dagger}}.$
Accordingly, the regret can be bounded as
\begin{align*}
&\text{Regret}(\mathcal{A},B)\leq \mathbb{E} \left[ \sum_{t=1}^{T_{\mathcal{A},B}} c_{A_t, t} + 1 \right] \frac{\mu^*_{i^\dagger}}{c_{i^\dagger}} 
- \mathbb{E} \left[ \sum_{t=1}^{T_{\mathcal{A},B}} r_{A_t, t} \right]
\\&
= \mathbb{E} \left[ \sum_{t=1}^{T_{\mathcal{A},B}} \left( \frac{\mu^*_{i^\dagger}}{c_{i^\dagger}} c_{A_t, t} - r_{A_t, t} \right) \right]
+ \frac{\mu^*_{i^\dagger}}{c_{i^\dagger}}\\
& \le \mathbb{E} \left[ \sum_{t=1}^{T_{\mathcal{A},B}} \left( \frac{\mu^*_{i^\dagger}}{c_{i^\dagger}} c_{A_t} - \mu^*_{A_t} \right) \right] + \sum_{i\in [N]} \nu_i \mu^*_i 
+ \frac{\mu^*_{i^\dagger}}{c_{i^\dagger}}\\
& = \sum_{i \neq i^\dagger} c_i \Delta_i \mathbb{E}[N_{i,T_{\mathcal{A},B}}] 
+ \frac{\mu^*_{i^\dagger}}{c_{i^\dagger}} +  \sum_{i\in [N]} \nu_i \mu^*_i, 
\end{align*}
where the inequality is due to 
\begin{align*}
\mathbb{E}\left[\sum_{t=1}^{T_{\mathcal{A},B}}r_{A_t,t}\right] &= \mathbb{E}\left[\sum_{i=1}^K\sum_{l=0}^{N_{i,T_{\mathcal{A},B}}}\mu_{i}(l)\right]
    \\&= \mathbb{E}\left[\sum_{i=1}^K\sum_{l=0}^{N_{i,T_{\mathcal{A},B}}}\mu_{i}^*\right] - \nu_i\mu_i^* + \sum_{l=0}^{\nu_i}\mu_i(l)
    \\&\ge \mathbb{E}\left[\sum_{i=1}^K\sum_{l=0}^{N_{i,T_{\mathcal{A},B}}}\mu_{i}^*\right] - \nu_i\mu_i^*. {\qed}
\end{align*} 
\end{proof}

We next only focus on the expected selecting time of each suboptimal worker. 
Define two events that $F_i = \{ \tau_i \ge \nu_i\}$ and $D_i=\{ \tau_i\le \nu_i+\omega\} $. $F_i$ implies that the $i$-th change point can only be detected by the algorithm after the change occurs. Denote $F_i^c$ and $D_i^c$ as the complementary event of $F_i$ and $D_i$. This leads to the following result:
\begin{lemma}
	Let $\gamma = 2\sqrt{\frac{2}{\omega} (14+\frac{12}{|\omega -1|})^2 \ln(B)}$, then $\mathbf{P}(F_iD_i)> 1-\frac{1}{B}$.
\end{lemma}
\begin{proof}
    $\mathbf{P}(F_iD_i) = \mathbf{P}(v_i < \tau_1 < v_i + \omega)$, which means that if a change occurs, the algorithm detects it within the window $\omega$. By inequality (8) in~\cite{xia2024llm},
    \begin{align*}
        \mathbf{P}(F_iD_i) &= \mathbf{P}(|\hat{\mu}_{w_1,i}(t+1) - \hat{\mu}_{w_2,i}(t+1)|>\frac{\gamma}{2}) \\&> 1 -\exp\{-\frac{\omega(\frac{\gamma}{2})^2}{2(14+\frac{12}{|\omega-1|})^2} \}\nonumber.
    \end{align*}
    Then substitute $\gamma = 2\sqrt{\frac{2}{\omega} (14+\frac{12}{|\omega -1|})^2 \ln(B)}$ to $1 -\exp\{-\frac{\omega(\frac{\gamma}{2})^2}{2(14+\frac{12}{|\omega-1|})^2} \}$, we have $ \mathbf{P}(F_iD_i) > 1 -\frac{1}{B}$.
\end{proof}

Then, we introduce two notations as follows,
\begin{align*}
	T_0 = \lfloor \frac{2B}{c_{\min}} \rfloor, \quad \quad N_0 = 2\log B\cdot  \frac{1}{{\min}^2\{
		\frac{\Delta_i (c_i)^2}{8(\mu^*_i + c_i)}, \frac{c_i}{4}
		  \}}.
\end{align*}
\begin{lemma}
For any $t\ge T_0$, we have $\mathbf{P}(B_{t+1}\ge 0)\le \exp(-2(B-tc_{\min})^2/t)$.
\end{lemma}
\begin{proof}
Recall that the cost incurred at round $t$ is $c_t$. When $t \ge T_0$ we have $\mathbb{E}[c_1] + \cdots + \mathbb{E}[c_t] \geq t \cdot c_{\min} \geq 2B > B$.
Accordingly, it holds that
\begin{align*}
\mathbf{P}(B_{t+1} \geq 0) &= \mathbf{P}(c_1 + \cdots + c_t \leq B) \\
&\leq \mathbf{P}(c_1 + \cdots + c_t - \mathbb{E}[c_1] - \cdots - \mathbb{E}[c_t] \\&\qquad\qquad\leq B - \mathbf{E}[c_1] - \cdots - \mathbb{E}[c_t])
\\
&\leq \exp \{- \frac{2(B - \mathbb{E}[c_1] - \cdots - \mathbb{E}[c_t])^2}{t} \}
\\ &
\leq \exp \{-\frac{2(B - t c_{\min})^2}{t} \},
\end{align*}
where the second inequality holds due to the Hoeffding inequality.
\end{proof}

Here, denote that algorithm CATI-UCB as $\mathcal{C}$, we are ready to have a regret bound for our algorithm:
\begin{theorem}
    Let $\gamma = 2\sqrt{\frac{2}{\omega} (14+\frac{12}{|\omega -1|})^2 \ln(B)}$, the regret of our algorithm CATI-UCB can be bounded by :
\begin{align*}
    \text{Regret}(\mathcal{C},B)&\le \sum_{i \neq i^\dagger} c_i \Delta_i\cdot  
(N_0 + \frac{4T_0}{B^2} + \mathcal{F}(B,c_{\min}) + 2\omega + \\&\; \nu_i +\frac{\frac{2B}{c_{\min}} + 1 +\mathcal{F}(B,c_{\min})}{B}
) +  \frac{\mu^*_{i^\dagger}}{c_{i^\dagger}} +  \sum_{i\in [N]} \nu_i \mu^*_i,
\end{align*}
where \begin{align}
T_0 &= \lfloor \frac{2B}{c_{\min}} \rfloor,\qquad N_0 = \frac{2\log B}{{\min}^2\left\{
 \frac{\Delta_i (c_i)^2}{8(\mu^*_i + c_i)}, \frac{c_i}{4}
 \right\}},\; \text{and}\nonumber \\
&\mathcal{F}(B,c_{\min}) = \exp\left\{\frac{-2{(B-1)}^2}{\frac{2B}{c_{\min}}+1}\cdot\frac{1}{1-\exp{\frac{-4{(B-1)}c_{\min}}{\frac{2B}{c_{\min}}+1}}}\right\} 
 \nonumber \\
&+ \exp \left\{-\frac{c_{\min} (B - c_{\min})^2}{B} \right\} + \sum_{\ell=T_0+1}^\infty \exp \left\{-\ell (c_{\min})^2 \right\} \nonumber
\end{align}

\end{theorem}
\textbf{Remark:} The term $\mathcal{F}(B,c_{\min})$ tends to zero as $B$ tends to infinity, $N_0$ is a logarithmic relationship with respect to $B$, and the remaining terms are constant terms, therefore, we can get the regret upper bound of CATI-UCB of $O(\log(B))$, which implies a fast convergence rate.
\begin{proof}
    The right side of Lemma~\eqref{lemma 2} shows that the regret of our algorithm is related to the number of selects for each suboptimal $i$, it can be characterized by
    \begin{align*}
\mathbb{E}[N_{i,T_{\mathcal{C},B}}]
&= \mathbb{E}[N_{i,T_{\mathcal{C},B}}-\nu_i] + \nu_i
\\& \leq \mathbb{E}[N_{i,T_{\mathcal{C},B}} - \tau_i | F_iD_i] + \mathbb{E}[\tau_i-\nu_i|F_iD_i] \\&\qquad\qquad+ \mathbb{E}[T_{\mathcal{C},B}|F^c_i D^c_i ] \cdot(1-\mathbf{P}(F_i D_i))+ \nu_i\\
	&\leq \underbrace{\mathbb{E}[N_{i,T_{\mathcal{C},B}} - \tau_i | F_iD_i]}_{\textbf{(i)}} +
	\underbrace{\mathbb{E}[\tau_i - \nu_i | F_iD_i]}_{\textbf{(ii)}} \\&\qquad\qquad+
	 \underbrace{\mathbb{E}[T_{\mathcal{C},B}](1-\mathbf{P}(F_iD_i))}_{\textbf{(iii)}}+\nu_i,
\end{align*}
where the first inequality is due to $N_{i,T_{\mathcal{C},B}} \le T_{\mathcal{C},B}$ and $\mathbf{P}(F_i D_i)\le 1$.

To bound part (i), it can be decomposed as follows,
\begin{align*}
	&\mathbb{E}[N_{i,T_{\mathcal{C},B}} - \tau_i | F_iD_i] \le \mathbb{E}[N_{i,T_{\mathcal{C},B}}  | F_iD_i]
 \\ 
 \le& \mathbb{E}\left[\sum^{T_0}_{t=1}\mathbf{I}\{ A_t = i, F_i  D_i\}\right]+\mathbb{E}\left[\sum^{\infty}_{T_0+1}\mathbf{I}\{ B_t\ge 0  \} \right]
 \\
\le & N_0 + \tau_i + \mathbb{E}\left[\sum^{T_0}_{t=1}\mathbf{I}\{ A_t = i, N_{i,t}\ge N_0 + \tau_i, F_i  D_i\}\right]\\&\qquad\qquad+\mathbb{E}\left[\sum^{\infty}_{T_0+1}\mathbf{I}\{ B_t\ge 0  \} \right].
\end{align*}
For notational convenience, denote $E_t$ as the event $\mathbf{I}\{ A_t = i, N_{i,t}\ge N_0+\nu_i, F_i  D_i\}$. Since $\mu^*_{i^\dagger} / c_{i^\dagger}  = \mu^*_i / c_i + \Delta_i$, if $E_t$ holds, which means that worker $i$ is selected at time $t$, then at least one of the following two events must happen:
\begin{align*}
	E^1_t: & \frac{\hat{\mu}_{i^\dagger}}{\hat{c}_{i^\dagger}}\le \frac{ \mu^*_{i^\dagger}} {c_{i^\dagger}} - \frac{\Delta_i}{2}, N_{i,t}\ge N_0+\tau_i, F_i  D_i, \\ 
	E^2_t: &  \frac{\hat{\mu}_{i}}{\hat{c}_{i}}> \frac{ \mu^*_{i}} {c_{i}} + \frac{\Delta_i}{2},  N_{i,t}\ge N_0+\tau_i, F_i  D_i.
\end{align*}
Otherwise, let $\{E^1_t\}^c$ and $\{E^2_t\}^c$ denote the complements event of $E^1_t$ and $E^2_t$. If both occur simultaneously, then:
\begin{align*}
    \{E^1_t\}^c \Rightarrow \frac{\hat{\mu}_{i^\dagger}}{\hat{c}_{i^\dagger}}& > \frac{ \mu^*_{i}} {c_{i}} +\Delta_i - \frac{\Delta_i}{2} 
    = \frac{\mu_i^*}{c_i^*} + \frac{\Delta_i}{2} \ge \frac{\hat{\mu}_{i}}{\hat{c}_{i}},
\end{align*}
where the last inequality corresponds to $\{E^2_t\}^c$. Thus, we have $\frac{\hat{\mu}_{i^\dagger}}{\hat{c}_{i^\dagger}} > \frac{\hat{\mu}_{i}}{\hat{c}_{i}}$, which contradicts the fact that worker $i$ was selected in round $t$.

By the concentration inequality (set $\delta = 1/B$), 
\begin{align*}
	&\mathbf{P}\left (\hat{\mu}_{i^\dagger} > \mu^*_{i^\dagger}|N_{i,t} \ge N_0 +\tau_i \right) \ge 1-\frac{1}{B^2},\\
	& \mathbf{P}\left (\hat{c}_{i^\dagger} < c_{i^\dagger}|N_{i,t} \ge N_0 +\tau_i \right) \ge 1-\frac{1}{B^2},
\end{align*} 
which gives $\mathbf{P}\left( E^1_t  \right) \le \frac{2}{B^2}$.
Similarly, with probability at least $1-\frac{2}{B^2}$,
\begin{align*}
\frac{\hat{\mu}_i}{\hat{c}_i} &\le \frac{\mu^*_i+2\epsilon_{i,t}}{\max\{c_i -2\epsilon_{i,t},\eta \}} 
\overset{(a)}{\le} 	\frac{\mu^*_i+2\epsilon_{i,t}}{c_i -2\epsilon_{i,t}}   = \frac{\mu^*_i}{c_i} + \frac{2\epsilon_{i,t}(\mu^*_i+c_i)}{c_i(c_i-2\epsilon_{i,t})}
\\&\overset{(b)}{\le} \frac{\mu^*_i}{c_i} + \frac{4\epsilon_{i,t}(\mu^*_i+c_i)}{(c_i)^2} \overset{(c)}{\le}  \frac{\mu^*_i}{c_i}  + \frac{\Delta_i}{2}.
\end{align*}
where the inequalities (a), (b), and (c) hold since $N_{i,t}\ge N_0$ ensures that 
$2\epsilon_{i,t}\le 0.5 c_i$ and $\epsilon_{i,t}\le \frac{\Delta_i(c_i)^2}{8(\mu^*_i+c_i)}$. Therefore, $\mathbf{P}(A^2_t)\le \frac{2}{B^2}$ ,
and thus
\begin{align*}
	&\mathbb{E}[N_{i,T_{\mathcal{C},B}} - \tau_i | F_iD_i] \\ 
    \le& N_0 + \sum^{T_0}_{t=N_0+\tau_i+1}\mathbf{I}\{ A_t\}+\sum^{\infty}_{T_0+1}\mathbf{I}\{ B_t\ge 0  \} \\
	 \le& N_0 + \frac{4T_0}{B^2} + \sum^{\infty}_{T_0+1}\mathbf{I}\{ B_t\ge 0  \} + \nu_i + w.
\end{align*}
Accroding to the Lemma 4, for any $t\ge T_0$, we have $\mathbf{P}(B_{t+1}\ge 0)\le \exp(-2(B-tc_{\min})^2/t)$. Now we are going to bound $\sum_{t=T_0+1}^\infty \mathbf{P}(B_t \geq 0)$. It holds that
\begin{align*}
&\sum_{t=T_0+1}^\infty \mathbf{P}(B_t \geq 0) \le\sum_{t=T_0}^\infty \mathbf{P}(B_{t+1} \geq 0)
\\
=& \exp \left\{-\frac{c_{\min} (B - c_{\min})^2}{B} \right\} + \sum_{t=T_0+1}^\infty \mathbf{P}(B_{t+1} \geq 0)
\\
=& X + \sum_{\ell=1}^\infty \exp \left\{-\frac{2 (B - T_0\cdot c_{\min}-\ell \cdot c_{\min})^2}{\ell+T_0} \right\}
\\
 \leq & X + \sum_{\ell=1}^\infty \exp \left\{-\frac{2 (B-c_{\min}+\ell\cdot c_{\min})^2}{\ell+T_0} \right\}
\\
\leq & X + \sum_{\ell=1}^\infty \exp \left\{-\frac{2(B- c_{\min}+\ell \cdot c_{\min})^2}{\frac{2B}{c_{\min}}+ \ell} \right\} \\
 \leq & X + \sum_{\ell=1}^{T_0} \exp \left\{-\frac{2(B- c_{\min}+\ell \cdot c_{\min})^2}{\frac{2B}{c_{\min}}+ \ell} \right\} \\&\quad+ \sum_{\ell=T_0+1}^{\infty} \exp \left\{-\frac{2(B-c_{\text{min}}+\ell \cdot c_{\min})^2}{\frac{2B}{c_{\min}}+ \ell} \right\}
\\
\leq & X + Z + \sum_{\ell=T_0+1}^\infty \exp \left\{-\ell (c_{\min})^2 \right\},
\end{align*}
where the second equality substitute $\exp \left\{-\frac{c_{\min} (B - c_{\min})^2}{B} \right\}$ to X, the second and third inequalities are due to the definition of $T_0$, the forth inequality uses $Z$ to substitute $\sum_{\ell=1}^{T_0} \exp \left\{-\frac{2(B- c_{\min}+\ell \cdot c_{\min})^2}{\frac{2B}{c_{\min}}+ \ell} \right\}$, and this term can be bounded by $\exp\left\{\frac{-2{(B-1)}^2}{\frac{2B}{c_{\text{min}}}+1}\cdot\frac{1}{1-\exp{\frac{-4{(B-1)}c_{\text{min}}}{\frac{2B}{c_{\text{min}}}+1}}}\right\}$. We use $\mathcal{F}(B,c_{\text{min}})$ to bound the right side of the above inequality. Please note that $\mathcal{F}(B,c_{\text{min}})$ tends to zero as $B$ tends to infinity.

Combining all things together we obtain for any $i \neq i^\dagger$,
\begin{equation}
	\mathbb{E}[N_{i,T_{\mathcal{C},B}} - \tau_i | F_iD_i]\le N_0 + \frac{4T_0}{B^2} + \mathcal{F}(B,c_{\min}) + \omega + \nu_i.\label{regret_term1}
\end{equation}
For the part (ii), it can be bounded by 
\begin{equation}\label{regret_term2}
	\mathbb{E}[ \tau_i - \nu_i | F_iD_i] \le \omega.
\end{equation}
According to the Lemma 3, the part (iii) can be bound by:
\begin{align}\label{regret_term3}
    \mathbb{E}[T_{\mathcal{C},B}](1-\mathbf{P}(F_iD_i)) &\le \frac{T_0 + \mathbb{E}\left[\sum^{\infty}_{T_0+1}\mathbf{I}\{ B_t\ge 0  \} \right]}{B}\nonumber
    \\&\le \frac{\frac{2B}{c_{\min}} + 1 +\mathcal{F}(B,c_{\min})}{B}.
\end{align}

Combining~\eqref{regret_term1},~\eqref{regret_term2} and~\eqref{regret_term3}, we bound $\mathbb{E}[N_{i,T_c,B}]$ as:
\begin{align*}
    \mathbb{E}[N_{i,T_{\mathcal{C},B}}] &\le N_0 + \frac{4T_0}{B^2} + \mathcal{F}(B,c_{\min}) + 2\omega \\&\qquad\qquad+ \nu_i+ \frac{\frac{2B}{c_{\min}} + 1 +\mathcal{F}(B,c_{\min})}{B}.
\end{align*}

By Lemma~\ref{lemma 2}, the regret of CATI-UCB can be bound by:
\begin{align*}
   & \text{Regret}(\mathcal{C},B)\le \sum_{i \neq i^\dagger} c_i \Delta_i \mathbb{E}[N_{i,T_{\mathcal{C},B}}] 
+ \frac{\mu^*_{i^\dagger}}{c_{i^\dagger}} +  \sum_{i\in [N]} \nu_i \mu^*_i \\
&\qquad \le \sum_{i \neq i^\dagger} c_i \Delta_i\cdot  
(N_0 + \frac{4T_0}{B^2} + \mathcal{F}(B,c_{\min}) + 2\omega + \nu_i+ \\&\qquad \frac{\frac{2B}{c_{\min}} + 1 +\mathcal{F}(B,c_{\min})}{B}
) +  \frac{\mu^*_{i^\dagger}}{c_{i^\dagger}} +  \sum_{i\in [N]} \nu_i \mu^*_i. \qedhere
\end{align*}
\end{proof}

\textbf{Remark 1: Extension to Multiple Task Types.}
CATI-UCB also extends to multiple task types. Instead of maintaining one set of statistics for each worker, we maintain separate statistics for each $(\text{worker}, \text{task type})$ pair.
Specifically, for each worker $i$ and task type $m$, we maintain:
(i) the type-dependent participation count $N_{i,m}$,
(ii) the type-dependent reward-learning curve parameters used in the online linear prediction module,
(iii) the type-dependent cost estimate, and
(iv) the corresponding change-detection statistics for convergence identification.
Then, when a task of type $m_t$ arrives at round $t$, CATI-UCB applies exactly the same cost-aware UCB rule, but conditioned on the statistics associated with type $m_t$.
After observing the reward and cost, only the selected $(i,m_t)$ pair is updated.
Hence, the core mechanism of CATI-UCB remains unchanged; the extension simply refines the estimation granularity from worker-level to worker-type-level.

This extension is algorithmically straightforward, but its statistical efficiency decreases as the number of task types grows. Under a fixed budget, each pair is observed less frequently, so its confidence radius, $O\!\left(\sqrt{\frac{\log B}{N_{i,m}}}\right),$
shrinks more slowly, while at least $2\omega$ observations are required before change detection can be performed. Thus, infrequent task types lead to slower reward and cost estimation, delayed convergence detection, and potentially larger regret. The current analysis therefore treats the heterogeneous-task extension as an applicability result for a finite and relatively small number of task types; deriving an explicit task-type-dependent regret bound is left for future work.

{\textbf{Remark 2: Extension to Combinatorial Selection.}
CATI-UCB can also be extended to a combinatorial setting in which up to $m$ workers are selected in each round.
In the simplest additive case, the algorithm still maintains worker-level reward-learning, cost-estimation, and change-detection statistics as in the single-worker setting, while the action in each round becomes the selection of a feasible subset.
Then, instead of choosing the single worker with the largest estimated reward-cost ratio, the platform chooses the subset of up to $m$ workers with the largest estimated reward-cost ratios.
If the optimal subset is denoted by $S^*$, then for any candidate subset $S$ one may define the gap $\Delta_S=\sum_{i\in S^*}\mu_i^*-\sum_{i\in S}\mu_i^*,$
which leads to a regret decomposition of the form $\mathrm{Regret}\le \sum_{S\neq S^*}\Delta_S\,\mathbb{E}[N_S]+O(1).$ Since every suboptimal subset must contain at least one suboptimal worker, the number of selections of such subsets can be bounded through the number of pulls of suboptimal workers.
Thus, the same high-level proof idea still yields a regret upper bound.}

\section{Experiment}
{In this section, we first evaluate CATI-UCB on synthetic data under piecewise linear and negative exponential reward functions to validate its performance and robustness.
We then present additional experiments on hyperparameter sensitivity, heterogeneous task types, and trace-driven real-world data to further demonstrate the effectiveness of the proposed framework.}

\subsection{Experimental Setup}
We first introduce the baselines and experimental settings, and then describe the evaluation metrics.

\subsubsection{Baselines} We compare our proposed algorithm against the following baseline algorithms and methods:
\begin{itemize}
    \item Primal-Dual~\cite{diaz2023flexible}: a primal-dual algorithm designed for constrained budget bandits problem.
    \item TIUCB~\cite{xia2024llm}: a model selection algorithm designed for time-increasing bandits problem.
    \item Budget-TIUCB: a time-increasing bandit algorithm that takes into account the constraint budget during the process, which is adapted based on~\cite{xia2024llm}.
    \item $\text{UCB}_c$~\cite{heyden2024budgeted}: A bandit algorithm that maximizes total reward under a budget constraint by dividing the reward’s upper bound by the cost’s lower bound.
\end{itemize}

\subsubsection{Parameter Setting and Metric.} For our proposed CATI-UCB algorithm, we set the change detection window $\omega = 16$ and threshold $\gamma = 0.2$ in all experiments. To introduce the variability, random noise is uniformly sampled from $[-0.05,0.05]$ for the synthetic simulations and is added to both the realized reward and cost at each round. All experiments are conducted under budget constraints $B \in \{50,100,200,400\}$, and each setting is repeated 20 times with independent random seeds to ensure robustness and reliability. {
The reported regret is computed with respect to the same offline oracle under the corresponding budget setting, i.e.,
\[
\mathrm{Regret}(A,B)=R^*(B)-\mathbb{E}[R_{A,B}],
\]
where $R^*(B)$ denotes the optimal expected cumulative reward achieved by an oracle that knows the reward trajectories and expected costs in advance.
Hence, the regret values quantify the gap between each online policy and the best achievable expected utility under the same budget.
We also report the mean reward, mean cost, reward-cost ratio, stopping time, and optimal-selection percentage for interpretability. According to the Lemma 1, the optimal strategy is to allocate all budget to the worker with highest reward-cost ratio.


\subsection{Piecewise Linear Function Data}
\begin{figure}[th]
    \centering
    \subfigure[Reward Functions]{
        \begin{minipage}[b]{0.49\linewidth}
        \includegraphics[width=0.99\textwidth]{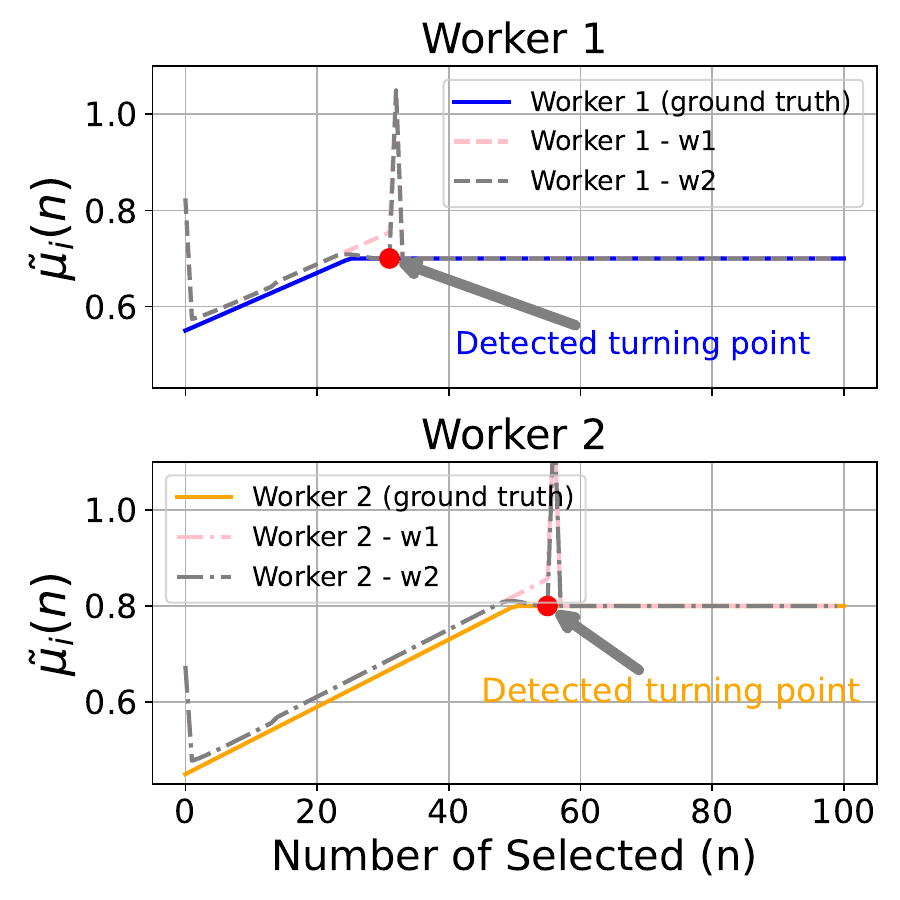}
        \end{minipage}
    \label{fig:piece_reward_fun}
    }
    \hspace{-6mm}
    \subfigure[Cumulative Regret]{
        \begin{minipage}[b]{0.49\linewidth}            \includegraphics[width=0.99\textwidth]{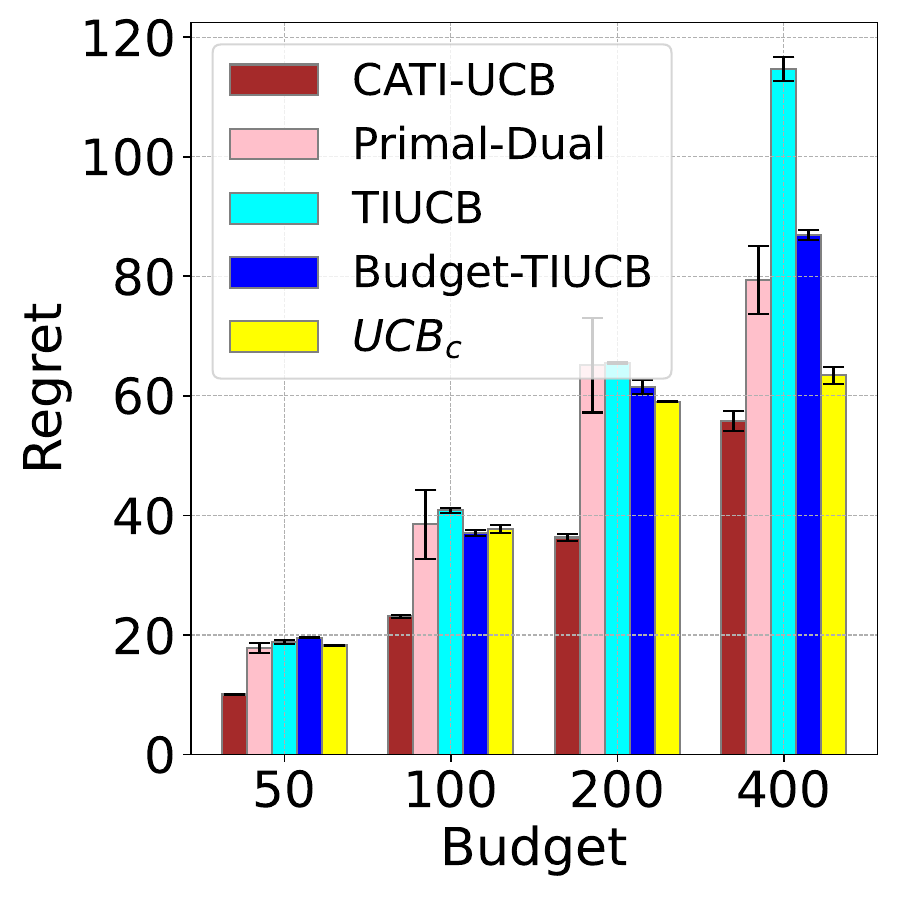}
        \end{minipage}
        \label{fig:piece_regret}
    }
    \caption{Performance of Piecewise Linear Function}
    \label{fig:Performance of 2-Arm}
\end{figure}
\begin{table}
\centering
\caption{\centering Statistics of the Piecewise Linear Function (Budget = 200)}
\begin{tabular}{|c|c|c|c|c|c|}
\hline
     & $\bar{r}$ & $\bar{c}$ & $\frac{\bar{r}}{\bar{c}}$ & $\bar{\tau}$ & (\%)opt \\ \hline
CATI-UCB      & 0.741                    & 0.602                    & 1.231                                                                        & 357.4             & 69.66  \\ \hline
Primal-Dual  & 0.749                    & 0.682                    &1.097                                                                        & 314.8             & 58.57  \\ \hline
TIUCB         & 0.749                    & 0.684                    & 1.095                                                                        & 314.2             & 58.62  \\ \hline
Budget-TIUCB & 0.753                   & 0.676                    & 1.114                                                                       & 318.0             & 62.26  \\ \hline
$\text{UCB}_c$       & 0.755                    & 0.671                    & 1.125                                                                        & 320.6             & 55.95  \\ \hline
\end{tabular}
\label{table-1}
\end{table}

We consider a synthetic setting where each worker’s reward increases linearly and then plateaus. As shown in Fig~\ref{fig:piece_reward_fun}, Worker~1 starts at $\mu_1(0) = 0.55$ and reaches $\mu_1^\star = 0.70$ after $\nu_1 = 25$ selects, while Worker~2 starts at $\mu_2(0) = 0.45$ and surpasses Worker~1 by reaching $\mu_2^\star = 0.80$ after $\nu_2 = 50$ selects. For any $i\in\{1,2\}$, the reward function is $\mu_i(t) = \mu_i(0) + \frac{\mu_i^\star - \mu_i(0)}{\nu_i} t$ for $t < \nu_i$ and $\mu_i(t) = \mu_i^\star$ for $t \ge \nu_i$.

To evaluate the change-point detection module in \textsc{CATI-UCB}, we also visualize the estimated reward curves under two window sizes (\texttt{w1}, \texttt{w2}). When the difference between these two estimates exceeds a detection threshold, the algorithm detects a change point, which is marked by the red dot. After detection, the algorithm resets the selecting statistics, resulting in a sharp spike in the estimated reward due to setting the count $N_{i,t} = 1$. As more observations accumulate, the estimate stabilizes and gradually converges to the true post-change value. A similar pattern is observed for Worker~2, showing the robustness of our change-point detection scheme.

Fig~\ref{fig:piece_regret} reports the cumulative regret with 2 workers as budget increases. \textsc{CATI-UCB} consistently outperforms all baselines and achieves sublinear regret.  Except for TIUCB, all other algorithms can achieve sublinear regret because they are also designed to balance reward and cost. Note that none of these algorithms have theoretical guarantees under the increasing-then-converge with budget constraint bandit problem. While $\text{UCB}_c$ performs well due to cost-awareness, it fails to model reward dynamics and underperforms our method. TIUCB performs worst, as it disregards budget constraints and tends to select high-cost workers with large empirical rewards, exhausting budget prematurely. 

We further report mean reward $\bar{r}$, mean cost $\bar{c}$, reward-cost ratio $\bar{r}/\bar{c}$, stopping time $\tau$, and optimal selecting percentage (\%opt) for $B=200$ in Table~\ref{table-1}. Table~\ref{table-1} shows that TIUCB has the largest mean costs for selecting the worker of these algorithms, which leads to the smallest termination time. \textsc{CATI-UCB} achieves the best $\bar{r}/\bar{c}$, the longest task horizon, and the highest optimal selection rate, showing its ability to balance reward and cost while adapting to worker quality.


\subsection{Negative Exponential Function Data}
Beyond the piecewise linear setting, prior studies have suggested that a worker’s learning curve can also be well captured by a negative exponential function \cite{wang2017recommending}. To validate the robustness of our algorithm under alternative reward structures, we conduct additional experiments where each worker’s expected reward follows the negative exponential function. 

We use two sets of synthetic reward functions to simulate models with increasing-then-converging performance patterns. The formats of two functions are as follows:
\begin{align}
    F_{\text{exp}} = \{f(t) = c(1-e^{-at})\} \quad \text{, and}\nonumber \\
    F_{\text{poly}} = \{f(t) = c(1-b{(t+b^{1/\rho})}^{-\rho})\}.\nonumber
\end{align}
These two families of functions are able to represent the general increasing-then-converging pattern of different shapes~\cite{metelli2022stochastic,xia2024llm}, where functions originating from $F_{\text{exp}}$ exhibit a rapid increase before converging, while those from $F_{\text{poly}}$ may display much slower growth rates.
\begin{figure}[th]
    \centering
    \subfigure[Reward Functions]{
        \begin{minipage}[b]{0.49\linewidth}
        \includegraphics[width=0.99\textwidth]{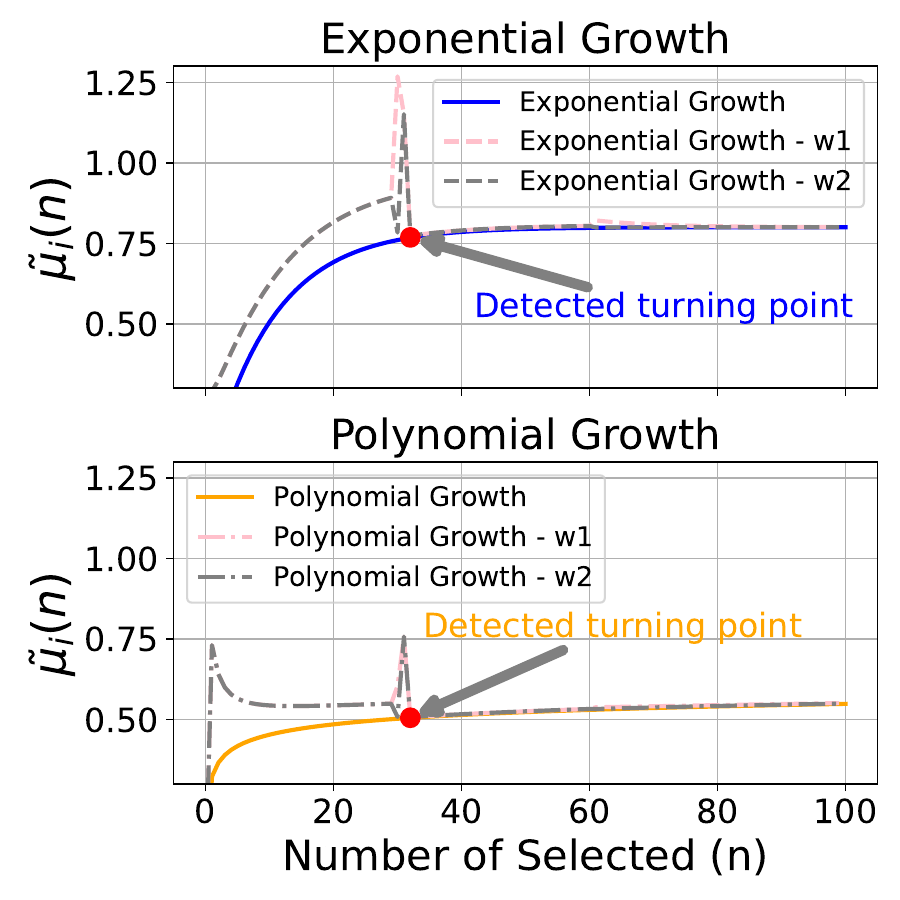}
        \end{minipage}
    \label{fig:2arm_reward_fun}
    }
    \hspace{-6mm}
    \subfigure[Cumulative Regret]{
        \begin{minipage}[b]{0.49\linewidth}            \includegraphics[width=0.99\textwidth]{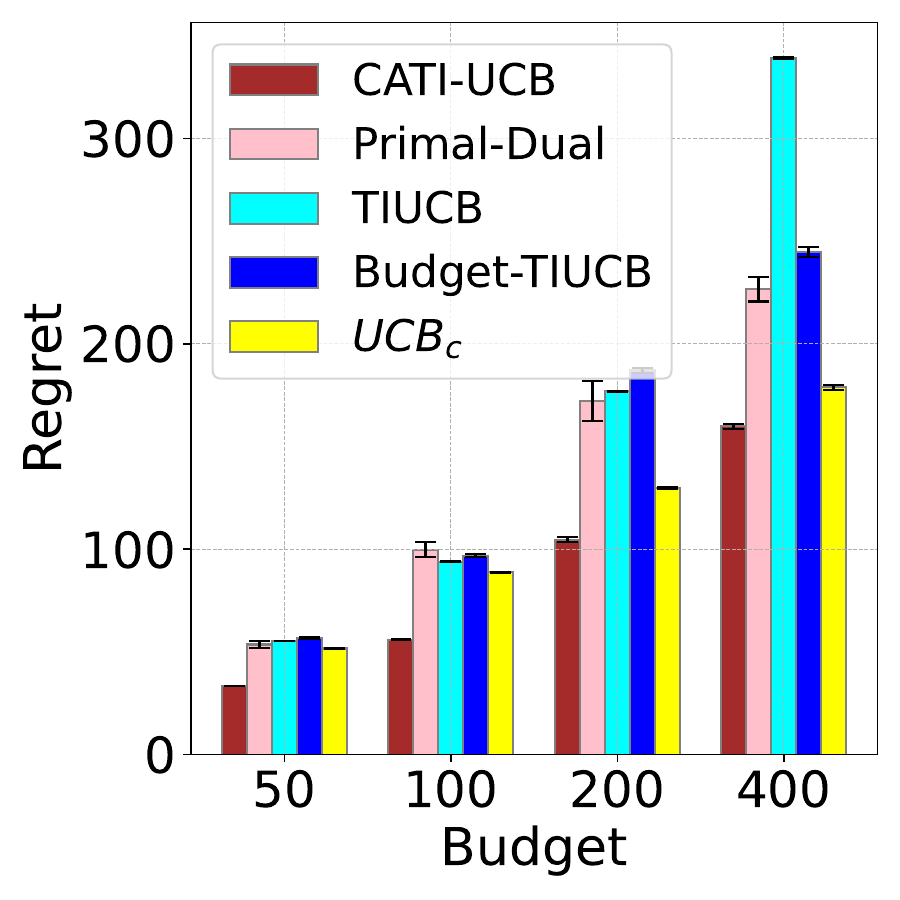}
        \end{minipage}
        \label{fig:2arm_regret}
    }
    \caption{Performance of Negative Exponential Function}
    \label{fig:Performance of 2-Arm}
\end{figure}

\begin{table}
\caption{\centering Statistics of The Negative Exponential Function (Budget = 200)}
\centering
\begin{tabular}{|c|c|c|c|c|c|}
\hline
     & $\bar{r}$ & $\bar{c}$ & $\frac{\bar{r}}{\bar{c}}$ & $\bar{\tau}$ & (\%)opt \\ \hline
CATI-UCB      & 0.584                    & 0.356                    & 1.643                                                                        & 562.2             & 85.80  \\ \hline
Primal-Dual  & 0.645                    & 0.494                    &1.305                                                                        & 404.8             & 51.43  \\ \hline
TIUCB         & 0.652                    & 0.508                    & 1.282                                                                        & 393.0             & 47.63  \\ \hline
Budget-TIUCB & 0.671                   & 0.545                    & 1.231                                                                       & 367.2             & 38.77  \\ \hline
$\text{UCB}_c$       & 0.601                    & 0.396                    & 1.518                                                                        & 504.8             & 75.79  \\ \hline
\end{tabular}
\label{table-2}
\end{table}
\begin{figure}[th]
    \centering
    \subfigure[Reward Functions]{
        \begin{minipage}[b]{0.49\linewidth}
        \includegraphics[width=0.99\textwidth]{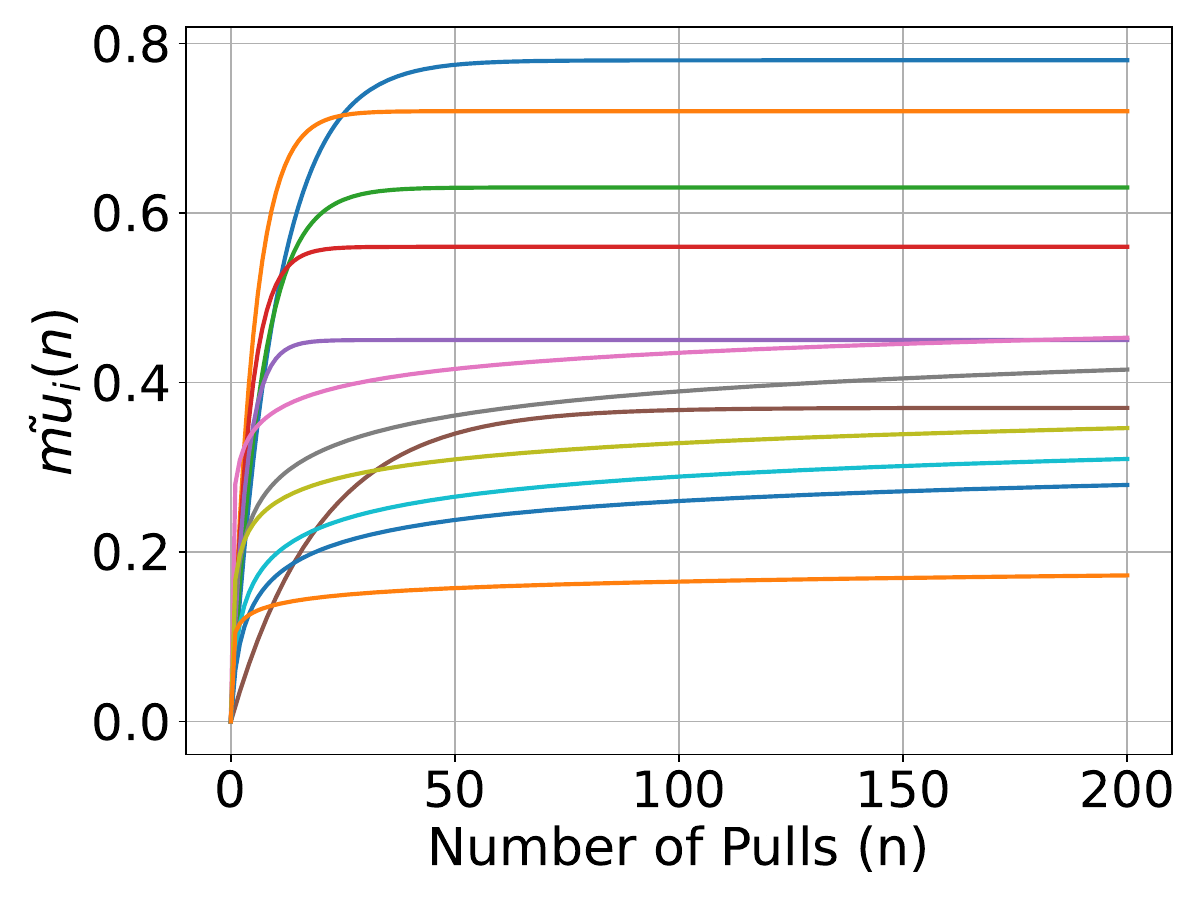}
        \end{minipage}
    \label{fig:12arm_reward_fun}
    }
    \hspace{-6mm}
    \subfigure[Cumulative Regret]{
        \begin{minipage}[b]{0.49\linewidth}            \includegraphics[width=0.99\textwidth]{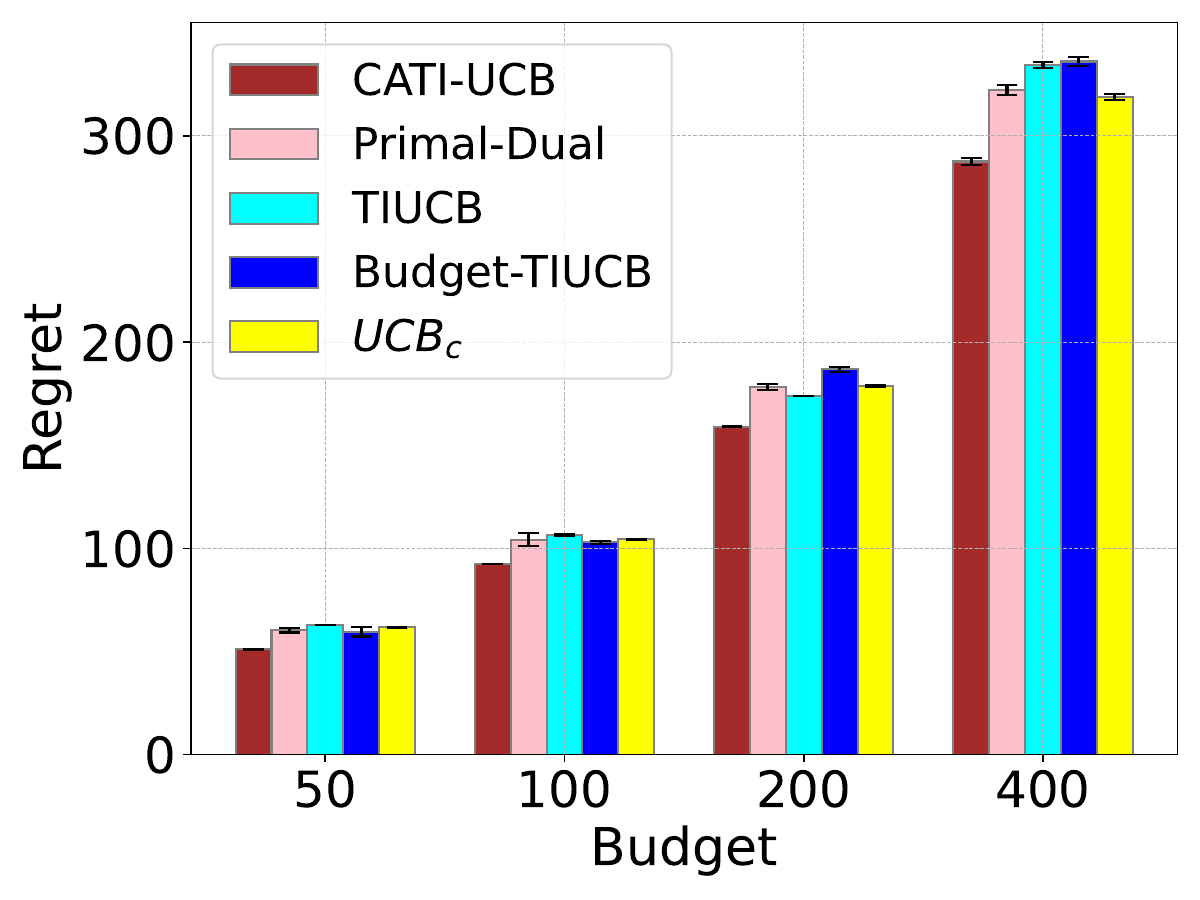}
        \end{minipage}
        \label{fig:12arm_regret}
    }
    \caption{Performance of 12-worker Bandits}
    \label{fig:Performance_of_12-arm}
\end{figure}
\begin{table}[!htbp]
    \caption{\centering Statistics of the 12-worker Time increasing Bandit (Budget = 200)}
\centering
\begin{tabular}{|c|c|c|c|c|c|}
\hline
     & $\bar{r}$ & $\bar{c}$ & $\frac{\bar{r}}{\bar{c}}$ & $\bar{\tau}$ & (\%)opt \\ \hline
CATI-UCB      & 0.457                    & 0.379                   & 1.204                                                                        & 784.0            & 34.29  \\ \hline
Primal-Dual  & 0.303                    & 0.273                    & 1.108                                                                        & 732.4             & 20.56  \\ \hline
TIUCB         & 0.434                    & 0.384                    & 1.131                                                                        & 520.2             & 9.57   \\ \hline
Budget-TIUCB & 0.271                    & 0.255                   &1.065                                                                        & 527.4             & 18.29  \\ \hline
$\text{UCB}_c$       & 0.418                    & 0.379                    & 1.106                                                                        & 528.4             & 30.02  \\ \hline
\end{tabular}
\label{table-3}
\end{table}
\subsubsection{2-worker Bandits} We conducted a two-worker bandit experiment with $a, b, c, \beta$ set to $0.0001, 1.0, 0.8$ and $0.1$, and the reward values shown in Fig~\ref{fig:2arm_reward_fun}. To test the trade-off between reward and cost, we assigned the exponential worker a high stable reward and 0.7 cost, and the polynomial worker a lower reward and 0.3 cost. Although the reward is no longer piecewise linear, CATI-UCB can still track the dynamics by fitting early-stage observations and applying change-point detection. As shown in the left panel, the algorithm identifies convergence points when estimates from two window sizes diverge, triggering a reset. This causes a brief spike due to $N_{i,t}=1$, which quickly stabilizes with more samples.

Fig.~\ref{fig:2arm_regret} and Table~\ref{table-2} also shows that our findings are consistent with the piecewise linear case: CATI-UCB consistently achieves the lowest cumulative regret across all budget levels, especially when the budget is large and the full learning dynamics unfold. Compared to $\text{UCB}_c$, which balances reward and cost but cannot adapt to time-evolving reward structures, CATI-UCB benefits from modeling the non-stationary growth and transitions. TIUCB again performs worst, as it does not account for budget constraints and tends to over-select costly workers based on short-term empirical averages. Due to the smoother reward transitions and slower initial growth, all algorithms exhibit slightly higher regret compared to the piecewise case. Nonetheless, CATI-UCB retains a clear advantage under this more realistic and challenging reward evolution.

\begin{figure}[th]
    \centering
    \subfigure[Window size $w$]{
        \begin{minipage}[b]{0.49\linewidth}
        \includegraphics[width=0.99\textwidth]{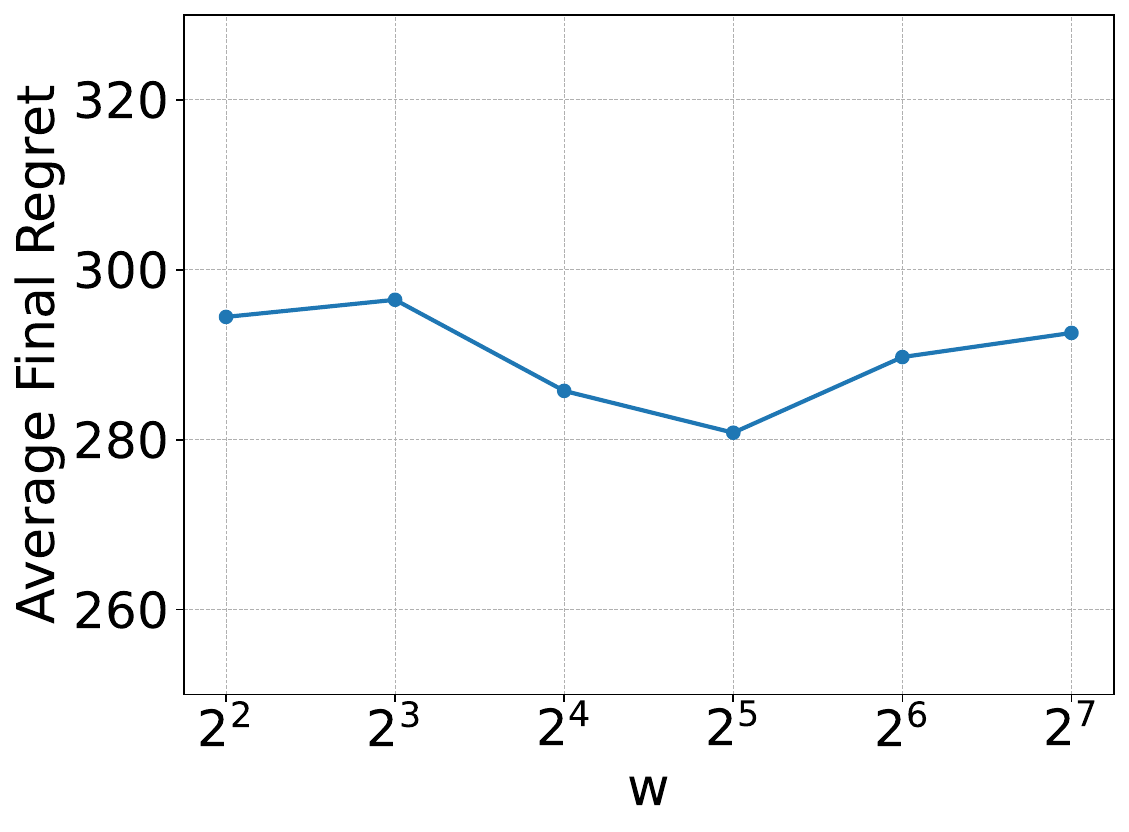}
        \end{minipage}
    \label{fig:window_size_change}
    }
    \hspace{-4mm}
    \subfigure[Threshold $\gamma$]{
        \begin{minipage}[b]{0.49\linewidth}            \includegraphics[width=0.99\textwidth]{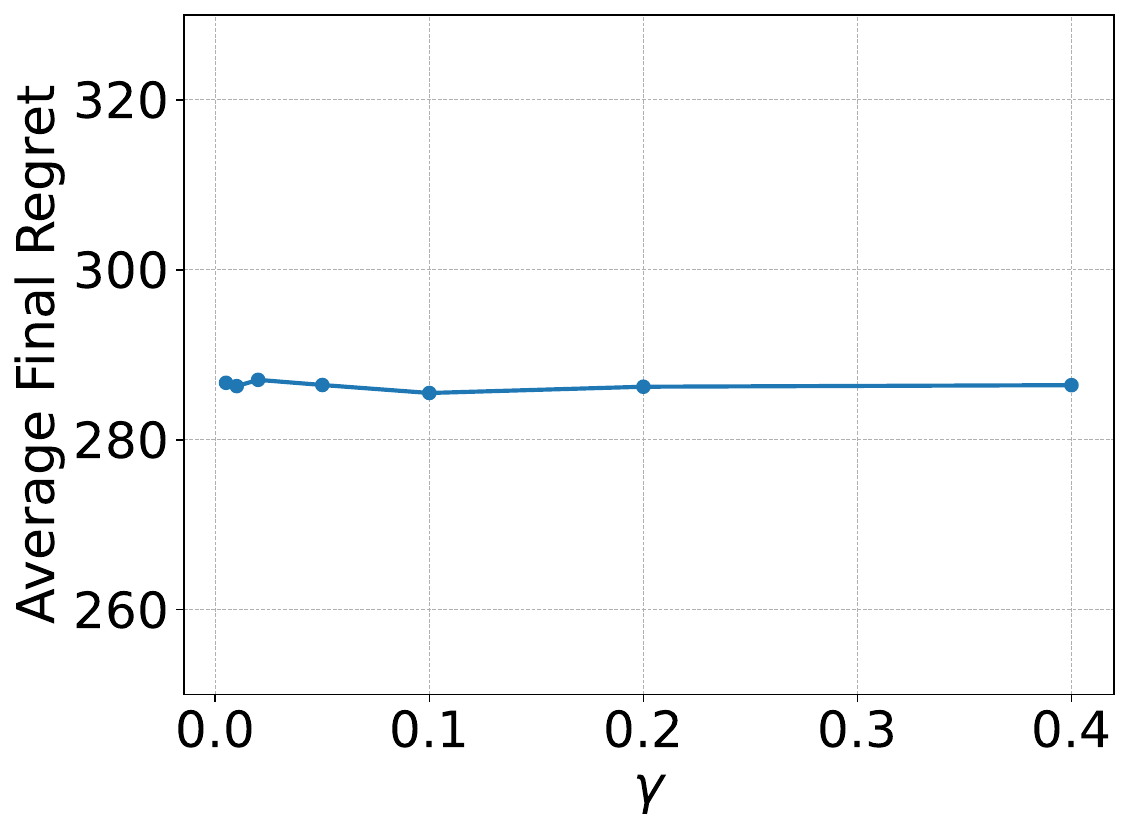}
        \end{minipage}
        \label{fig:threshold_change}
    }
    \caption{Sensitivity of CATI-UCB to Hyperparameters}
    \label{Sensitivity of CATI-UCB}
\end{figure}

\subsubsection{12-worker Bandits}
In order to generalize the effectiveness of our experimental results, we also compare the performance of each algorithm on 12 workers. Then the value of $a,c,\rho \in (0,1]$ and $b \in \mathbb{R} \ge 0$ are selected randomly and the results are shown in Fig~\ref{fig:12arm_regret} and Table~\ref{table-3}. As the number of workers increases, each algorithm needs to spend more cost to explore the worker, resulting in a larger regret, so the total regret is divided by 10 to better show the performance of each model. Table~\ref{table-3} shows that the selecting percentage of the optimal (\%opt) of each algorithm decreases than that of 2-worker bandit setting, further verifying that as the number of workers increases, the algorithm spends more cost to learn the suboptimal worker. However, CATI-UCB still achieves the highest percentage of selecting the optimal worker and the highest $\frac{\bar{r}}{\bar{c}}$. TIUCB achieves the smallest percentage of selecting the optimal worker and the smallest $\frac{\bar{r}}{\bar{c}}$, which is consistent with the CATI-UCB outperform all baselines in Fig~\ref{fig:12arm_regret}.

\subsection{Sensitivity to Change-Detection Hyperparameters}

We next examine the sensitivity of CATI-UCB to the two hyperparameters in the change-detection module, namely the window size $\omega$ and the threshold $\gamma$.
Unless otherwise specified, all sensitivity experiments in this subsection are conducted under the same \emph{12-arm synthetic setting} as in the main experiments. We report the average final regret over repeated runs.

Fig.~\ref{fig:window_size_change} shows the effect of varying the window size over $\omega \in \{2^2,2^3,\ldots,2^7\}$.
The result exhibits a clear non-monotonic pattern.
When $\omega$ is too small, the detector becomes overly sensitive to local fluctuations, which leads to unstable resets and higher regret.
When $\omega$ is too large, detection becomes delayed, so the algorithm adapts too slowly after the reward trend changes.
The best result in this experiment is achieved around $\omega=2^5$, while moderate window sizes such as $\omega\in\{16,32\}$ clearly outperform the extreme choices.
In the main experiments, we use $\omega=16$ as a representative moderate setting that performs competitively while preserving a shorter reaction horizon.

Fig.~\ref{fig:threshold_change} reports the effect of varying the threshold over $\gamma \in \{0.01,0.02,0.05,0.1,0.2,0.4\}$.
Compared with the window-size study, the regret varies only mildly across a broad range of threshold values, indicating that CATI-UCB is fairly robust to the exact choice of $\gamma$ once it is set in a reasonable interval.
This is also consistent with the algorithmic mechanism: since the detection decision is based on averaged statistics from two windows, the smoothing induced by $\omega$ already stabilizes the detector, making performance less sensitive to the precise threshold value.
In our experiment, moderate thresholds around $\gamma\in[0.1,0.2]$ provide slightly better performance, which is why we use $\gamma=0.2$ in the main text.}

\subsection{Heterogeneous Task Types}

To validate that CATI-UCB extends naturally beyond the homogeneous-task setting, we further consider a heterogeneous-task scenario with two task types.
In this setting, each worker has task-type-dependent reward dynamics, so the algorithm maintains separate reward-learning, cost-estimation, and change-detection statistics for each $(\text{worker}, \text{task type})$ pair, as discussed in the above remark on multiple task types.

We first consider a toy setting with $2$ workers and $2$ task types.
Fig.~\ref{fig:two_task_two_arm_reward_fun} shows representative type-specific reward trajectories under both exponential-growth and polynomial-growth patterns.
The figure also illustrates the two window-based prediction curves and the detected turning point for each worker--task pair.
We observe that CATI-UCB still tracks the reward evolution accurately and successfully detects the transition from the growth phase to the saturation phase for different task types. Then we consider a larger setting with $12$ workers and $2$ task types.
The corresponding type-specific reward trajectories are shown in Fig.~\ref{fig:two_task_12_arm_reward_fun}.
Compared with the $2$-worker case, the platform now needs to learn many more worker--task pairs, so the exploration burden is higher under the same budget.

Since, in the heterogeneous-task setting, the realized reward depends on the task-type sequence and each worker has type-dependent reward functions, we report cumulative reward instead of a single regret quantity.
Fig.~\ref{fig:Performance of heterogeneous-task} shows the cumulative reward for both settings.
CATI-UCB achieves the highest cumulative reward in both cases.
At the same time, the overall reward in the $12$-worker case is lower than that in the $2$-worker case under the same budget, since more budget must be spent on exploration to learn the task-type-dependent reward functions.
These results show that the proposed framework remains effective beyond the homogeneous-task case and continues to outperform the baselines in heterogeneous-task settings.

\begin{figure}[tb]
	\centering
	\includegraphics[width=0.9\linewidth]{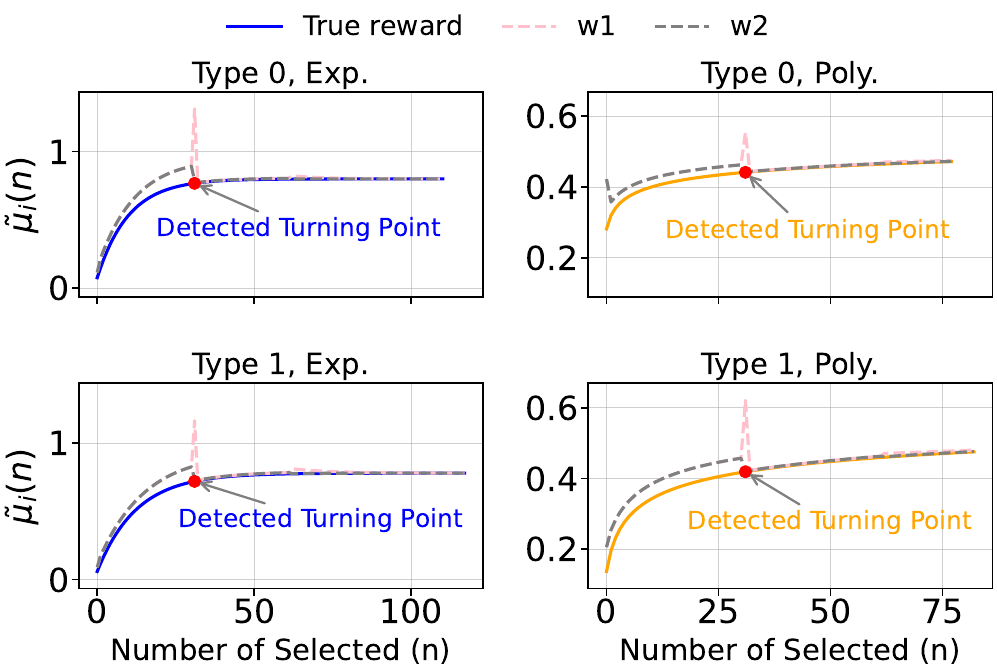}
	\caption{Type-Specific Reward Trajectories in the 2-Arm Setting} \label{fig:two_task_two_arm_reward_fun}
\end{figure}%

\begin{figure}[tb]
	\centering
	\includegraphics[width=0.9\linewidth]{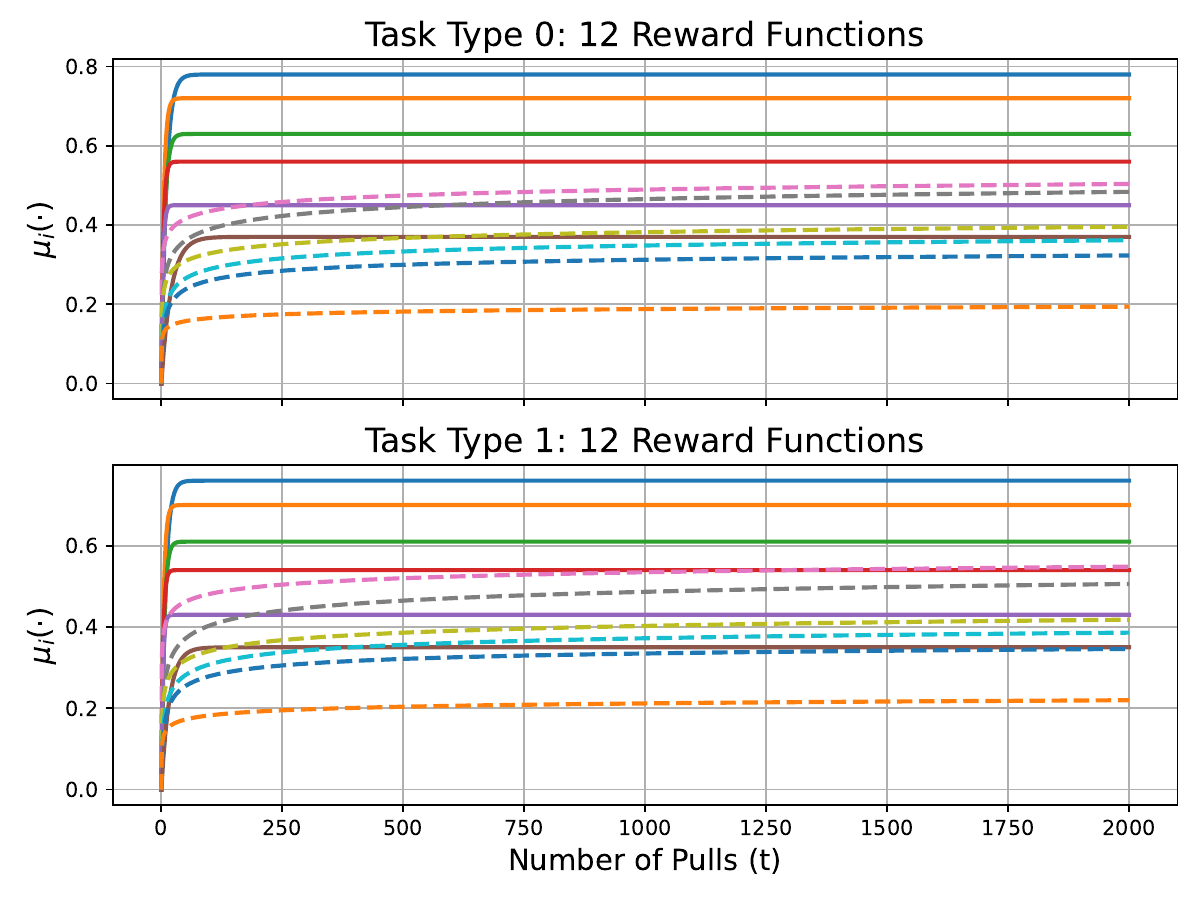}
	\caption{Type-Specific Reward Trajectories in the 12-Arm Setting} \label{fig:two_task_12_arm_reward_fun}
\end{figure}%

\begin{figure}[th]
    \centering
    \subfigure[2-Arm Setting]{
        \begin{minipage}[b]{0.49\linewidth}
        \includegraphics[width=0.99\textwidth]{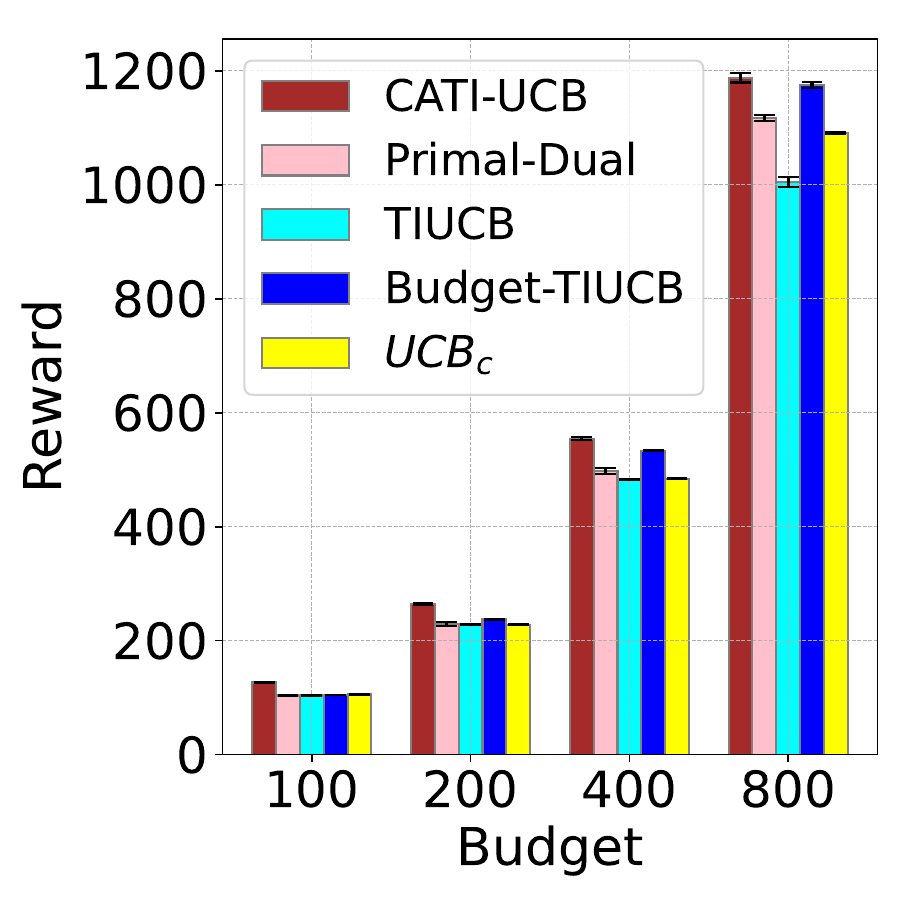}
        \end{minipage}
    \label{fig:two_type_2_arm_reward}
    }
    \hspace{-6mm}
    \subfigure[12-Arm Setting]{
        \begin{minipage}[b]{0.49\linewidth}            \includegraphics[width=0.99\textwidth]{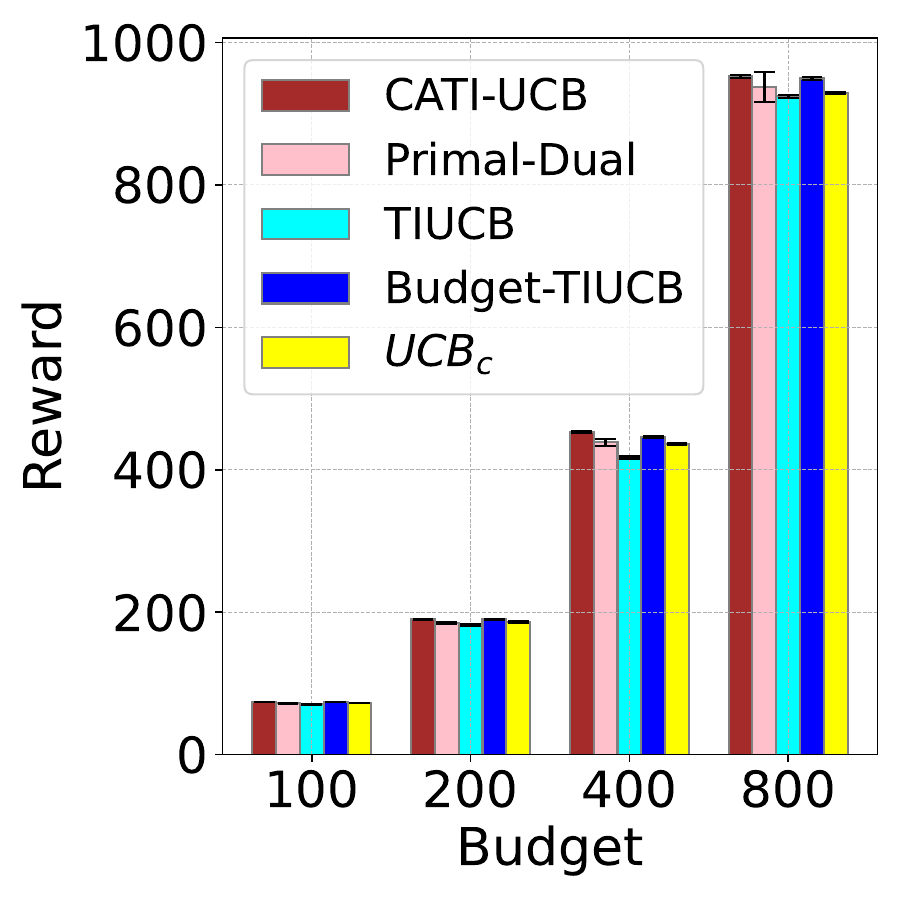}
        \end{minipage}
        \label{fig:two_type_12_arm_reward}
    }
    \caption{Cumulative Reward in the Heterogeneous Task-Types Setting}
    \label{fig:Performance of heterogeneous-task}
\end{figure}

\subsection{Trace-Driven Evaluation Based on a Real-World Dataset}

\begin{figure}[th]
    \centering
    \subfigure[Empirical Reward Trajectories]{
        \begin{minipage}[b]{0.49\linewidth}
        \includegraphics[width=0.99\textwidth]{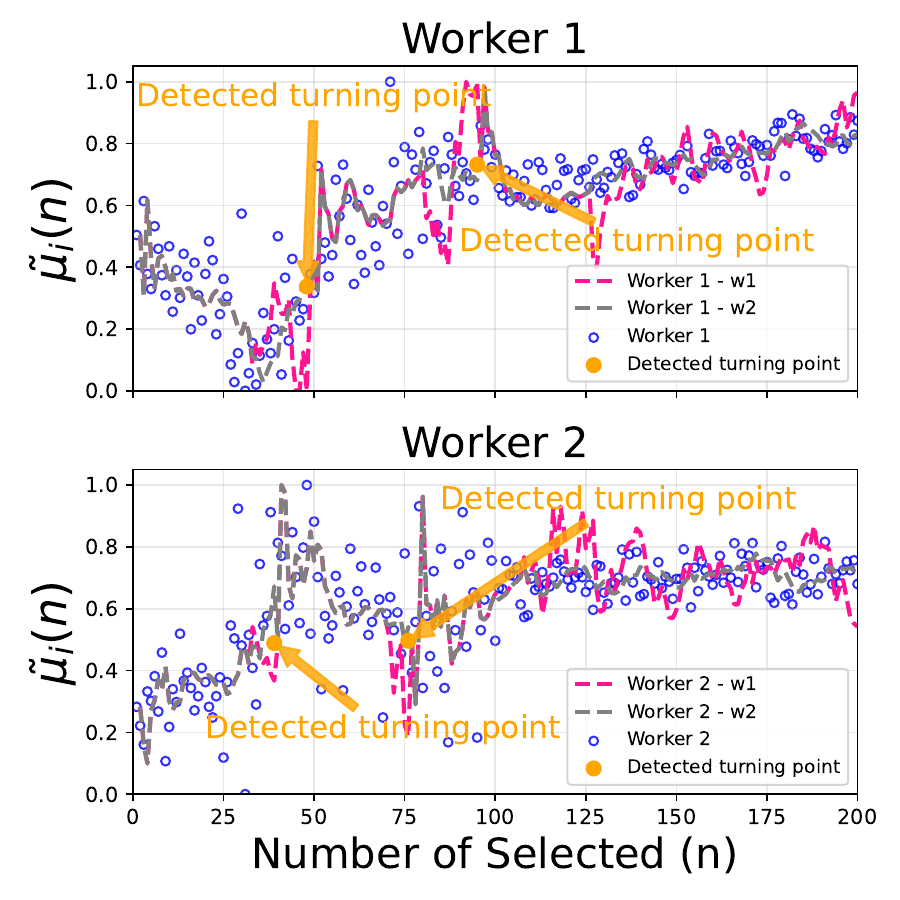}
        \end{minipage}
    \label{fig:real_world_dataset}
    }
    \hspace{-6mm}
    \subfigure[Cumulative Regret]{
        \begin{minipage}[b]{0.49\linewidth}            \includegraphics[width=0.99\textwidth]{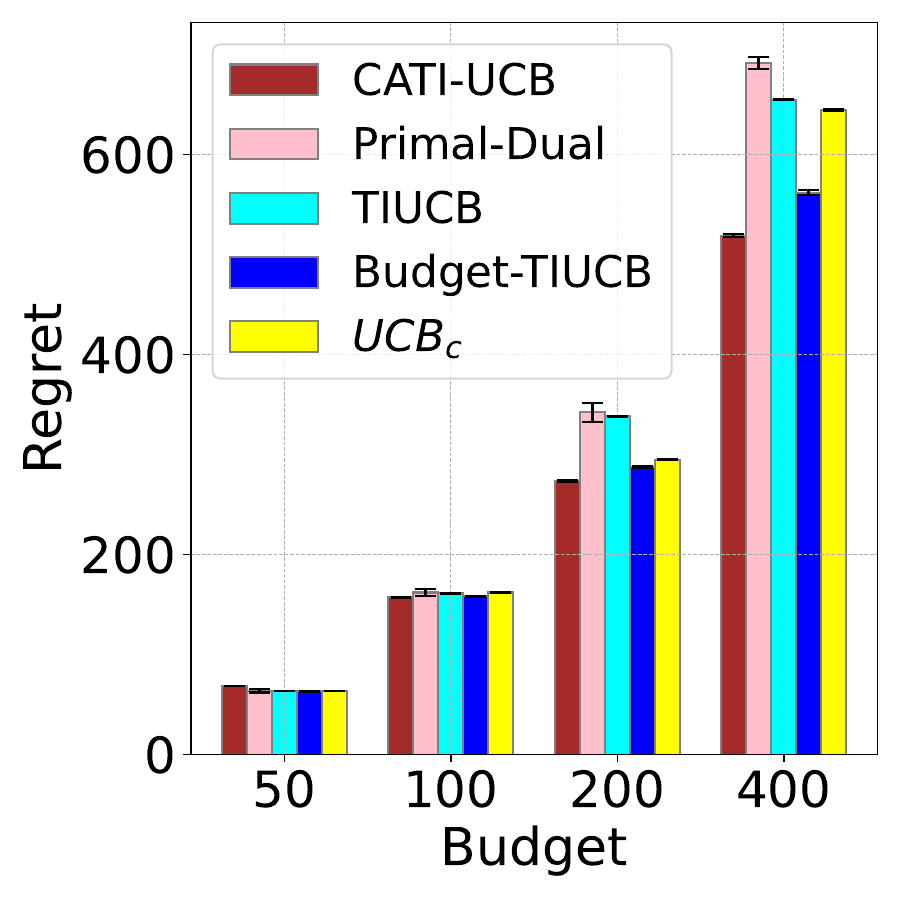}
        \end{minipage}
        \label{fig:real_dataset_regret}
    }
    \caption{Trace-Driven Evaluation Based on the Topcoder Dataset}
    \label{fig:Performance of 2-Arm}
\end{figure}

To complement the synthetic experiments, we construct a trace-driven benchmark based on a real-world Topcoder dataset~\cite{wang2017recommending}.
Since the original raw records are not publicly available, we follow the skill-improvement pipeline in~\cite{wang2017recommending} and build empirical reward trajectories from the published performance trends.
Specifically, we use the reported worker-performance curves to fit representative reward traces, and then sort each worker's scores chronologically to form an empirical reward trajectory.

Fig.~7(a) shows representative empirical reward trajectories for two workers.
The blue circles denote empirical reward observations, while the pink and gray dashed curves denote least-squares prediction trajectories under two detection windows.
The orange dots mark the turning points detected by CATI-UCB.
Compared with the synthetic setting, these real-data-driven traces are substantially noisier and may contain multiple local turning points, which makes convergence detection more challenging.

Fig.~7(b) reports the corresponding budgeted online-selection result.
CATI-UCB achieves the smallest regret for moderate and large budgets, indicating that the proposed framework remains effective when the reward process is derived from real worker-performance data.
At the smallest budget, the empirical reward sequence is still highly unstable and the detector has limited history, so CATI-UCB is slightly less favorable than some baselines.
As the budget increases and more sequential observations become available, CATI-UCB becomes consistently superior.
These trace-driven experiments thus provide direct evidence that CATI-UCB is not limited to purely synthetic environments.

\section{Conclusion}
In this paper, we studied budget-constrained task allocation under structured worker learning, where each worker's reward follows an unknown, increasing-then-converging trajectory with fixed costs. To address this, we proposed CATI-UCB, a cost-aware algorithm that estimates reward and cost, detects learning convergence via change-point detection, and adaptively allocates budget. We provide theoretical guarantees and show that CATI-UCB achieves sublinear regret and outperforms baselines across synthetic and real-world settings. An important direction for future work is to characterize the
regret dependence on the number and arrival frequencies of task
types, and to improve statistical scalability through parameter
sharing or contextual generalization across related task types.

\bibliographystyle{IEEEtran}
\bibliography{references}

\end{document}